%% file: main.tex
\documentclass{article}
\usepackage{ijcai24}

\usepackage{times}
\usepackage{soul}
\usepackage{url}
\usepackage{dsfont}
\usepackage[hidelinks]{hyperref}
\usepackage[utf8]{inputenc}
\usepackage[small]{caption}
\usepackage{graphicx}
\usepackage{amsmath}
\usepackage{amsthm}
\usepackage{booktabs}
\usepackage{algorithm}
\usepackage[noend]{algorithmic}
\usepackage[switch]{lineno}

\usepackage{booktabs}
\usepackage{xcolor}
\usepackage{amssymb,amsfonts}
\usepackage{bm}
\usepackage{stackengine}
\usepackage{makecell}
\usepackage{multirow}
\usepackage{subfigure}
\usepackage{adjustbox}
\usepackage{comment}
\usepackage{enumitem}

\theoremstyle{definition}

\newtheorem{lemma}{Lemma}
\newcommand{\ourmodel}{\text{ParIRL}}
\newcommand{\Real}{\mathbb{R}}

\newcommand{\StateSet}{S}
\newcommand{\ActionSet}{A}
\newcommand{\Dynamics}{P}
\newcommand{\RewardFunc}{r}
\newcommand{\Datasets}{{\mathbb{T}}}
\newcommand{\Dataset}{\mathcal{T}}

\newcommand{\PolicySet}{\Pi}
\newcommand{\ParetoPolicySet}{\Pi^*}
\newcommand{\mo}{\text{mo}}

\DeclareMathOperator*{\expectation}{\mathbb{E}}
\DeclareMathSymbol{*}{\mathbin}{symbols}{"03}

\newtheorem{theorem}{Theorem}

\title{Pareto Inverse Reinforcement Learning for Diverse Expert Policy Generation}

\author{
Woo Kyung Kim\and
Minjong Yoo\and
Honguk Woo\thanks{Honguk Woo is the corresponding author.}
\affiliations
Department of Computer Science and Engineering, Sungkyunkwan University\\
\emails
\{kwk2696, mjyoo2, hwoo\}@skku.edu
}

\begin{document}

\maketitle

\begin{abstract}
Data-driven offline reinforcement learning and imitation learning approaches have been gaining popularity in addressing sequential decision-making problems. Yet, these approaches rarely consider learning Pareto-optimal policies from a limited pool of expert datasets. This becomes particularly marked due to practical limitations in obtaining comprehensive datasets for all preferences, where multiple conflicting objectives exist and each expert might hold a unique optimization preference for these objectives.
%
In this paper, we adapt inverse reinforcement learning (IRL) by using reward distance estimates for regularizing the discriminator. This enables progressive generation of a set of policies that accommodate diverse preferences on the multiple objectives, while using only two distinct datasets, each associated with a different expert preference. 
In doing so, we present a Pareto IRL framework ($\ourmodel$) that establishes a Pareto policy set from these limited datasets. 
In the framework, the Pareto policy set is then distilled into a single, preference-conditioned diffusion model, thus allowing users to immediately specify which expert's patterns they prefer.
Through experiments, we show that $\ourmodel$ outperforms other IRL algorithms for various multi-objective control tasks, achieving the dense approximation of the Pareto frontier. 
We also demonstrate the applicability of $\ourmodel$ with autonomous driving in CARLA.
\end{abstract}

\section{Introduction}
\begin{figure}[t] 
    \centering
    \includegraphics[width=.96\linewidth]{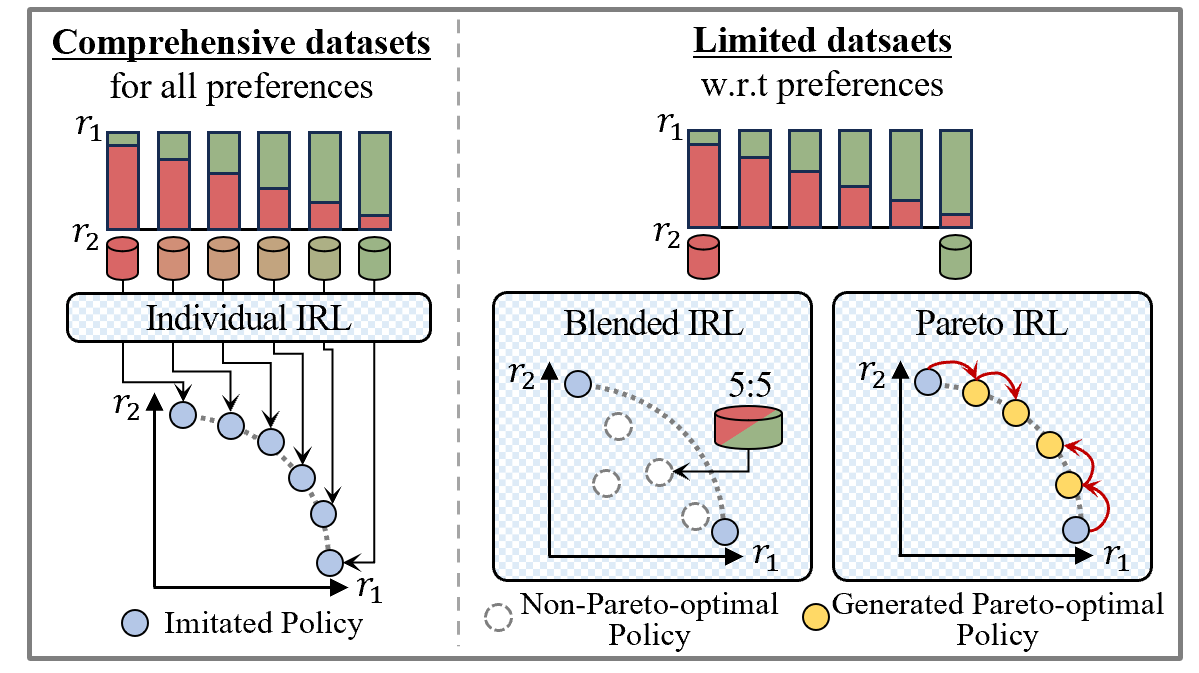}
    \caption{Data-driven Pareto policy set learning}
    \label{fig:intro}
\end{figure}
In decision-making scenarios, each expert might have her own preference on multiple, possibly conflicting objectives (multi-objectives).
Accordingly, learning Pareto-optimal policies in multi-objective environments has been considered essential and practical to provide users with a selection of diverse expert-level policies, which can cater their specific preferences (e.g.,~\cite{morl:pgmorl,morl:ppa}).
However, in the area of imitation learning, such multi-objective problem has not been fully explored due to the requirement for comprehensive expert datasets encompassing the full range of multi-objective preferences (e.g.,~\cite{moil:pdt}), which might be unattainable in real-world scenarios.
%


In the ideal scenario depicted on the left side of Figure~\ref{fig:intro}, having comprehensive expert datasets encompassing diverse multi-objective preferences enables the straightforward derivation of a Pareto policy set by reconstructing policies from each dataset. However, this is often not feasible in real-world situations where datasets might not represent all preferences. 
This common limitation is illustrated on the right side of Figure~\ref{fig:intro}. Here, one typically has access to only two distinct datasets, each reflecting different multi-objective preferences. In such limited dataset cases, a viable approach involves merging these datasets in varying proportions, followed by the application of imitation learning on each blended dataset. However, this approach often leads to a collection of non-Pareto-optimal policies, as demonstrated in Section~\ref{section:eval}.


In this paper, we address the challenges of multi-objective imitation learning in situations with strictly limited datasets, specifically focusing on Pareto policy set generation. 
Our goal is to derive optimal policies that conform with diverse multi-objective preferences, even in the face of limited datasets regarding these preferences. 
To do so, we investigate inverse reinforcement learning (IRL) and present a Pareto IRL ($\ourmodel$) framework in which a Pareto policy set corresponding to the best compromise solutions over multi-objectives can be induced. 
This framework is set in a similar context to conventional IRL where reward signals are not from the environment, but it is intended to obtain a dense set of Pareto policies rather than an individually imitated policy.

In {$\ourmodel$}, we exploit a recursive IRL structure to find a Pareto policy set progressively in a way that at each step, nearby policies can be derived between the policies of the previous step. Specifically, we adapt IRL using reward distance regularization; new policies are regularized based on reward distance estimates to be balanced well between  distinct datasets, while ensuring the regret bounds of each policy. This recursive IRL is instrumental in achieving the dense approximation of a Pareto policy set. 
%
Through distillation of the approximated Pareto policy set to a single policy network, we build a diffusion-based model, which is conditioned on multi-objective preferences. This distillation 
not only enhances the Pareto policy set but also integrates it into a single unified model, thereby facilitating the zero-shot adaptation to varying and unseen preferences.   

The contributions of our work are summarized as follows.
\begin{itemize}
    \item We introduce the $\ourmodel$ framework to address a novel challenge of imitation learning, Pareto policy set generation from strictly limited datasets.
    \item We devise a recursive IRL scheme with reward distance regularization to generate policies that extend beyond the datasets, and we provide a theoretical analysis on their regret bounds.
    \item We present a preference-conditioned diffusion model to further enhance the approximated policy set on unseen preferences. This allows users to dynamically adjust their multi-objective preferences at runtime. 
    \item We verify $\ourmodel$ with several multi-objective environments and autonomous driving scenarios,   
    demonstrating its superiority for Pareto policy set generation.
    \item  $\ourmodel$ is the first to tackle the data limitation problem for Pareto policy set generation within the IRL context. 
\end{itemize}

\section{Preliminaries and Problem Formulation}
\subsection{Background}
\noindent\textbf{Multi-Objective RL (MORL).}
A multi-objective Markov decision process (MOMDP) is formulated with multiple reward functions, each associated with an individual objective.
\begin{equation}
    (\StateSet,\ \ActionSet,\ \Dynamics,\ \mathbf{\RewardFunc},\ \Omega,\ f, \gamma)
\label{equ:momdp}
\end{equation}
Here, $s \in \StateSet$ is a state space, $a \in \ActionSet$ is an action space, $\Dynamics: \StateSet \times \ActionSet \times \StateSet \rightarrow [0,1]$ is a transition probability, and $\gamma \in [0, 1]$ is a discount factor.
MOMDP incorporates a vector of $m$ reward functions $\mathbf{\RewardFunc} = [r_1, ..., r_m]$ for $\RewardFunc: \StateSet \times \ActionSet \times \StateSet \rightarrow \Real$, a set of preference vectors $\Omega \subset \Real^m$, and a linear preference function $f(\mathbf{r}, \bm{\omega}) = \bm{\omega}^{\text{T}} \bm{r}$ where $\bm{\omega} \in \Omega$. 
The goal of MORL is to find a set of Pareto polices $\pi^* \in \ParetoPolicySet$ for an MOMDP environment, where $\pi^*$ maximizes scalarized returns, i.e.,  $\max _\pi \expectation_{a \sim \pi(\cdot|s)} [\sum_{t=1}^{H} \gamma^tf(\bm{r},\bm{\omega})]$.

\noindent\textbf{Inverse RL (IRL).} 
Given an expert dataset $\Dataset^*=\{\tau_i\}_{i=1}^{n}$, where each trajectory $\tau_i$ is represented as a sequence of state and action pairs $\{(s_t,a_t)\}_{t=1}^{T}$, IRL aims to infer the reward function of the expert policy, thus enabling the rationalization of its behaviors. 
Among many, the adversarial IRL algorithm (AIRL) casts IRL into a generative adversarial problem~\cite{irl:airl,irl:diffail} with such discriminator as
\begin{equation}
    \label{equ:disc}
    D(s,a,s') = \frac{\exp(\tilde{r}(s,a,s'))}{\exp(\tilde{r}(s,a,s')) + \pi(a|s)}
\end{equation}
where $s' \sim \Dynamics(s, a, \cdot)$ and $\tilde{r}$ is a inferring reward function.
The discriminator is trained to maximize the cross entropy between expert dataset and dataset induced by the policy via 
\begin{equation} \label{equ:disc:loss}
\begin{aligned}
    \max \ [ & \mathbb{E}_{(s,a)\sim \Dataset_{\pi}} [\log (1 - D(s,a,s'))] + \\
    & \mathbb{E}_{(s,a) \sim \Dataset^*} [ \log D(s,a,s')]]
\end{aligned}
\end{equation}
where $\Dataset_{\pi}$ is the dataset induced by learning policy $\pi$.
The generator of AIRL corresponds to $\pi$, which is trained to maximize the entropy-regularized reward function such as
\begin{equation} \label{equ:gen:loss}
\begin{aligned}
    \log(D(s,a,s')) - & \log(1 - D(s,a,s')) \\
    & = \tilde{r}(s,a,s') - \log\pi(a|s).
\end{aligned}
\end{equation}

\begin{figure}[t]
    \centering
        \includegraphics[width=.96\linewidth]{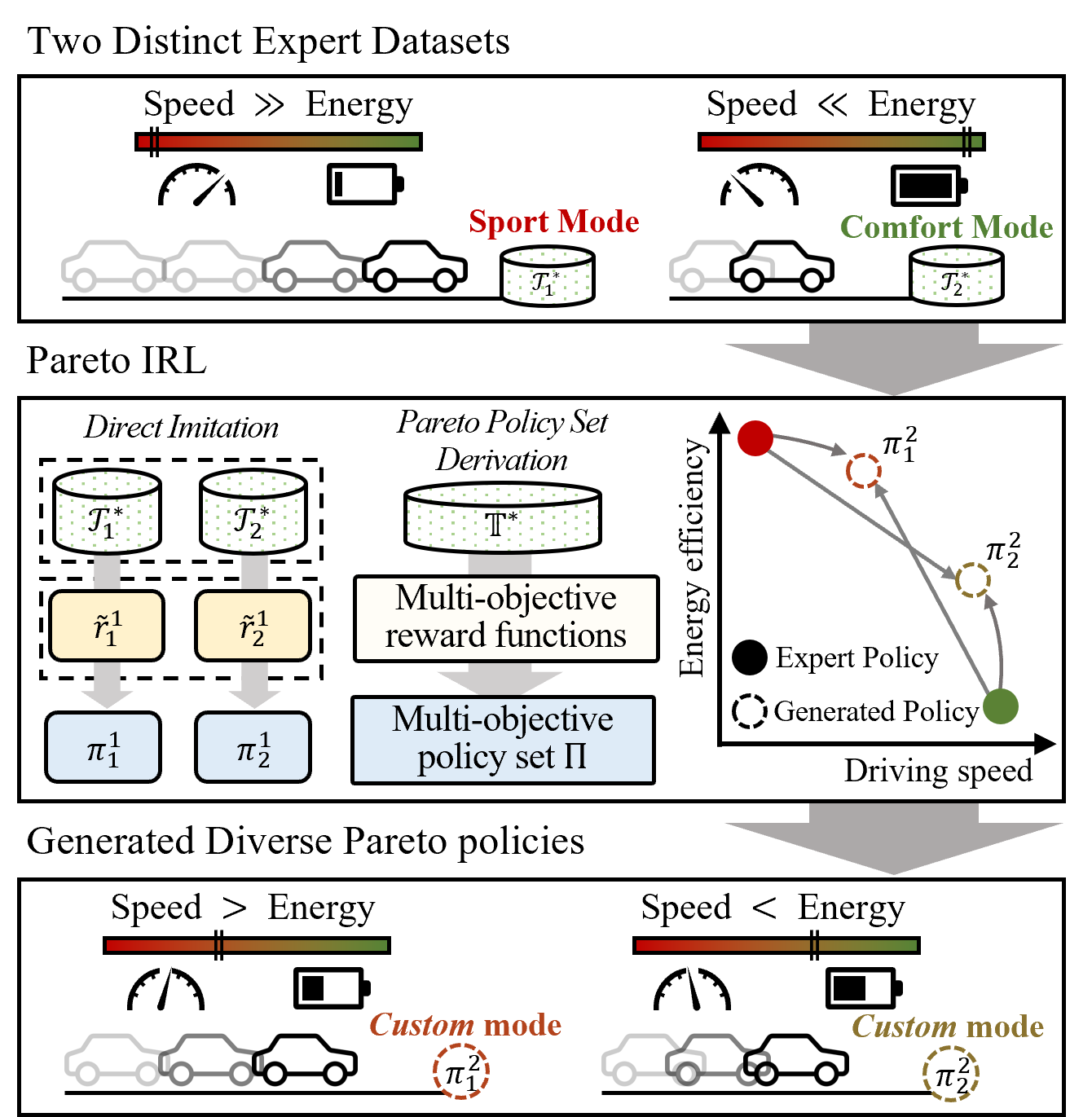}
    \caption{Concept of Pareto policy set generation: given two distinct expert datasets, each associated with a specific preference over multi-objectives (e.g., some expert prefers speed over energy efficiency, and vice versa), the Pareto IRL is to find a set of optimal compromise \textit{custom} policies, each of which can conform to a different preference.}
    \label{fig:problem}
\end{figure}

\subsection{Formulation of Pareto IRL}
We specify the Pareto IRL problem which derives a Pareto policy set from strictly limited datasets. 
%
Consider $M$ distinct expert datasets ${\Datasets}^* = \{\Dataset^*_{i}\}_{i=1}^{M}$ where each expert dataset $\Dataset^*_i$ is collected from the optimal policy on some reward function $r_\mo = \bm{\omega}_i^\text{T}\mathbf{r}$ with a fixed preference $\bm{\omega}_i \in \Omega$.
Furthermore, we assume that each dataset $\Dataset^*_i$  distinctly exhibits dominance on a particular reward function $r_i$.
In the following, we consider scenarios with two objectives ($M\!=\!2$), and later discuss the generalization for three or more objectives in Appendix A.4.

Given two distinct datasets, in the context of IRL, we refer to Pareto policy set derivation via IRL as Pareto IRL. 
Specifically, it aims at inferring a reward function $\tilde{r}$ and learning a policy $\pi$ for any preference $\bm{\omega}$ from the strictly limited datasets $\Datasets^*$. 
That is, when exploiting limited expert datasets in a multi-objective environment, we focus on establishing the Pareto policy set effectively upon unknown reward functions and preferences. 
%

%
Figure~\ref{fig:problem} briefly illustrates the concept of Pareto IRL, where a self-driving task involves different preferences on two objectives, possibly conflicting, such as driving speed and energy efficiency.
Consider two distinct expert datasets, where each expert has her own preference settings for the driving speed and energy efficiency objectives (e.g.,  $\Dataset^*_1$ and $\Dataset^*_2$ involve one dominant objective differently). 
While it is doable to restore a single useful policy individually from one given expert dataset, 
our work addresses the issue to generate a set of policies $\PolicySet$ which can cover a wider range of preferences beyond given datasets. 
The policies are capable of rendering optimal compromise returns, denoted by dotted circles in the figure, and they allow users to immediately select the optimal solution according to their preference and situation. 

For an MOMDP with a set of preference vectors $\bm{\omega} \in \Omega$, a vector of reward functions $\mathbf{r}$, and a preference function $f$ in~\eqref{equ:momdp}, Pareto policy set generation is 
to find a set of multi-objective policies such as
\begin{equation}
    \Pi = \{\pi\ |\ R_{f(\mathbf{r}, \bm{\omega})}(\pi) \geq R_{f(\mathbf{r}, \bm{\omega})}(\pi'), 
    \forall \pi', \exists\ \bm{\omega} \in \Omega \}
\end{equation}
for $M$ expert preference datasets $\{\Dataset^*_i\}_{i=1}^{M}$.
$R_{r}(\pi)$ represents returns induced by policy $\pi$ on reward function $r$. Neither a vector of true reward functions $\mathbf{r}$ is explicitly revealed, nor the rewards signals are annotated in the expert datasets, similar to conventional IRL scenarios. 

\color{black}
\begin{figure*}[t] 
    \centering
    \includegraphics[width=.94\textwidth]{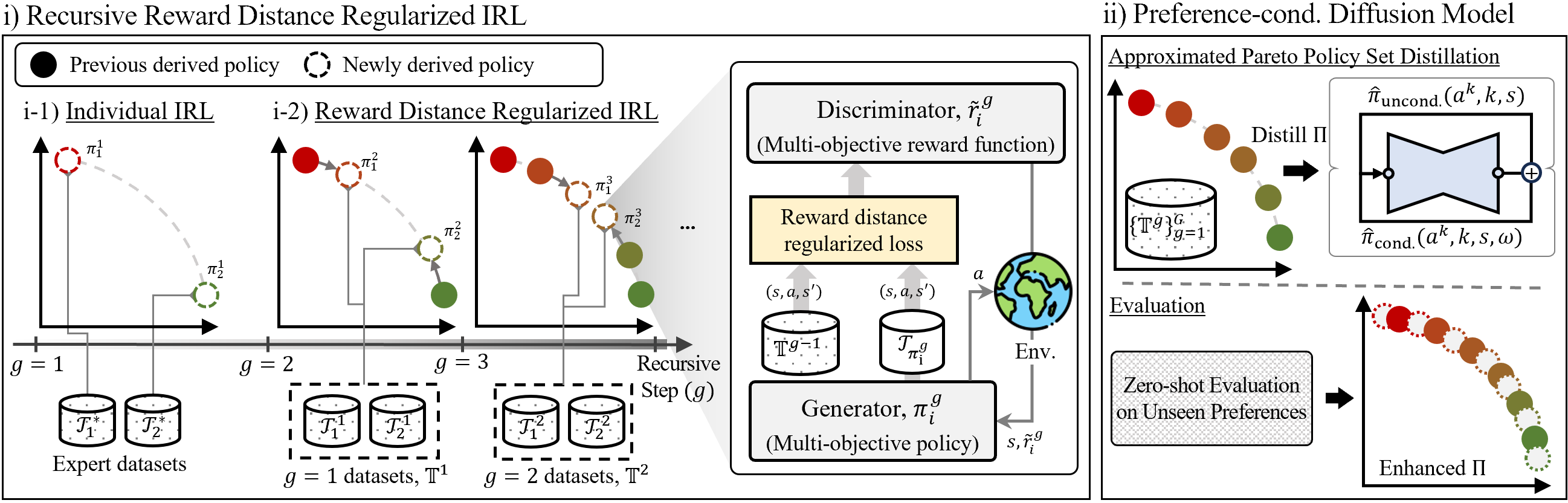}
    \caption{$\ourmodel$ framework: 
    In (i), policies in a Pareto set are recursively derived via reward distance regularized IRL. In (ii), the preference-conditioned diffusion model enhances the approximated Pareto policy set via distillation.
    }
    \label{fig:architecture}
\end{figure*}
\section{Our Framework}
To obtain a Pareto policy set from strictly limited datasets, we propose the $\ourmodel$ framework involving two learning phases: (\romannumeral 1) recursive reward distance regularized IRL,  (\romannumeral 2) distillation to a preference-conditioned model. 
%

In the first phase, our approach begins with direct imitation of the given expert datasets, and then recursively finds neighboring policies that lie on the Pareto front. 
Specifically, we employ the reward distance regularized IRL method that incorporates reward distance regularization into the discriminator's objective to learn a robust multi-objective reward function.
This regularized IRL ensures that the performance of the policy learned by the inferred multi-objective reward function remains within the bounds of the policy learned by the true reward function. 
By performing this iteratively, we achieve new useful policies that are not presented in the expert datasets, thus establishing a high-quality Pareto policy set.

%

In the second phase, we distill the Pareto policy set into a preference-conditioned diffusion model. The diffusion model encapsulates both preference-conditioned and unconditioned policies, each of which is associated with the preference-specific knowledge (within a task) and the task-specific knowledge (across all preferences), respectively.
Consequently, the unified policy model further enhances the Pareto policy set, rendering robust performance for arbitrary \textit{unseen} preferences in a zero-shot manner. It also allows for efficient resource utilization with a single policy network.


\subsection{Recursive Reward Distance Regularized IRL} \label{subsec:recursive_irl}
\noindent\textbf{Notation.} We use superscripts $g \in \{1,...,G\}$ to denote recursive step and subscripts $i \in \{1,2\}$ to denote $i$-th multi-objective policies derived at each recursive step $g$.
We consider two objectives cases in the following.

\noindent\textbf{Individual IRL.} As shown in the Figure~\ref{fig:architecture} (\romannumeral 1-1), the framework initiates with two separate IRL procedures, each dedicated to directly imitating one of the expert datasets.
For this, we adopt AIRL~\cite{irl:airl} which uses the objectives~\eqref{equ:disc:loss} and~\eqref{equ:gen:loss} to infer reward functions $\{\tilde{r}_{i}^{1}\}_{i=1}^{2}$ and policies $\{\pi_{i}^1\}_{i=1}^2$ from the individual expert dataset $\Dataset^*_i \in \Datasets^*$.
%

\noindent\textbf{Reward distance regularized IRL.}
Subsequently, as shown in the Figure~\ref{fig:architecture} (\romannumeral 1-2), at each recursive step $g \geq 2$, we derive new multi-objective reward functions $\{\tilde{r}_{i}^{g}\}_{i=1}^{2}$ and respective policies $\{\pi^g_i\}_{i=1}^2$ that render beyond the given datasets.
To do so, a straightforward approach might involve conducting IRL iteratively by blending the expert datasets at different ratios.
However, as illustrated in Figure~\ref{fig:ana:pareto}, the resulting policies tend to converge towards some weighted mean of datasets, rather than fully exploring non-dominant optimal actions beyond simple interpolation of given expert actions.
To address the problem, we present a reward distance regularized IRL on datasets $\Datasets^{g-1} \! = \! \{\Dataset_i^{g-1}\}_{i=1}^{2}$ collected from the policies derived at the previous step.
%
Given a reward distance metric $d(r, r')$, we compute the distance between the newly derived reward function $\tilde{r}^g_i$ and previously derived reward functions $\tilde{\mathbf{r}}^{g-1} = [\tilde{r}^{g-1}_1, \tilde{r}^{g-1}_2]$.
Further, we define target distances as a vector $\bm{\epsilon}^g_i = [\epsilon^g_{i,1}, \epsilon^g_{i,2}]$ to constrain each of the corresponding measured reward distances.
Then, we define a reward distance regularization term as
\begin{equation} 
    \label{equ:moirl:regul}
    \begin{split}
        I(\tilde{r}_{i}^{g},\tilde{\mathbf{r}}^{g-1}) =\sum\nolimits_{j=1}^2 \left( \epsilon^g_{i,j} - d(\tilde{r}_{i}^{g}, \tilde{r}_{j}^{g-1}) \right)^2
    \end{split}
\end{equation}
where the subscripts $i$ and $j$ denote the newly derived reward function and the previously derived one, respectively.
Finally, we incorporate~\eqref{equ:moirl:regul} into the discriminator objective~\eqref{equ:disc:loss} as
\begin{equation} \label{equ:moirl:loss}
\begin{aligned}
    \max \ & \mathbb{E}_{(s,a) \sim \mathcal{T}_{\pi_i^g}} [ \log ( 1 - D(s,a,s') )] + \\
    & \mathbb{E}_{(s,a) \sim \Datasets^{g-1}} [\log D(s,a,s')] - \beta \cdot I(\tilde{r}_i^g,{\tilde{\mathbf{r}}^{g-1}})
\end{aligned}
\end{equation}
where $\beta$ is a hyperparameter.
This allows the discriminator to optimize a multi-objective reward function for a specific target distance across datasets.
The reward distance regularized IRL procedure is performed twice with different target distances, to derive policies adjacent to each of the previously derived policies.
Furthermore, we fork the new regularized IRL procedure with the previously one (that is adjacent) to enhance the efficiency and robustness in learning.

The choice of the target distance is crucial,  
as the regret of a multi-objective policy is bounded under the reward distance~\eqref{equ:regbound_3}.
Thus, we set the sum of the target distances as small as possible.
As the reward distance metrics satisfy the triangle inequality $d(\tilde{r}_{1}^{g-1},\tilde{r}_{2}^{g-1}) \leq d(\tilde{r}_i^g,\tilde{r}_{1}^{g-1}) + d(\tilde{r}_i^g,\tilde{r}_{2}^{g-1})$, we limit the sum of target distances to 
\begin{equation}
    \label{equ:moirl:trdsum}
     \hat{\epsilon}^g_i = \sum\nolimits_{j=1}^2 \epsilon^g_{i,j} = d(\tilde{r}_{1}^{g-1},\tilde{r}_{2}^{g-1}).
\end{equation}
In practice, we assign a small constant value for one of the target distances $\epsilon^g_{i,i}$, while the other is  determined as $\hat{\epsilon}^g_i - \epsilon^g_{i,i}$. 
By doing so, we are able to effectively derive a new policy that is adjacent to one of the previous policies.

Any reward distance metric that guarantees the regret bounds of policy can be used for $\ourmodel$.
In our implementation, we adopt EPIC, also known as equivalent policy invariant comparison pseudometric~\cite{rw:epic}, which quantitatively measures the distance between two reward functions.
The learning procedure of recursive reward distance regularized IRL is summarized in Algorithm~\ref{algo:PaIRL}.
In Appendix A.4, we discuss the generalization of reward distance regularization for more than two objectives ($M \geq 3$).

\begin{algorithm}[t]
    \textbf{Input:} Expert preference datasets $\Datasets^{*}=\{\Dataset^{*}_1,\Dataset^{*}_{2}\}, \PolicySet = \emptyset$
    
    \begin{algorithmic}[1]
    \STATE {/* $1$-st step: individual IRL procedures */}
    \FOR{$i \leftarrow 1,...,2$}
        \STATE{Obtain $\tilde{r}_i^1$ and $\pi_i^1$ using $\Dataset^{*}_{i}$ through \eqref{equ:disc:loss} and \eqref{equ:gen:loss}}
        \STATE{Collect dataset $\Dataset^{1}_{i}$ by executing $\pi_{i}^{1}$}
        \STATE{Execute $\Datasets^1 \leftarrow \Datasets^1 \cup \Dataset_{i}^{1}$ and $\PolicySet \leftarrow \PolicySet \cup \{\pi_{i}^{1}\}$}       
    \ENDFOR    
    \STATE {/* $g$-th step: regularized IRL procedure */}
    \FOR{$g \leftarrow 2,...,G$}
        \FOR{$i \leftarrow 1,...,2$}
            \STATE{Initialize $\tilde{r}_i^g \leftarrow \tilde{r}_{i}^{g-1}$ and $\pi_i^g \leftarrow \pi_{i}^{g-1}$}
            \STATE{Set $\bm{\epsilon}^g_i = [\epsilon^g_{i,1}, \epsilon^g_{i,2}]$ based on \eqref{equ:moirl:trdsum}}
            \WHILE{not converge}
                \STATE{Update $\tilde{r}_i^g$ using $\bm{\epsilon}^g_i$ and regularized loss in \eqref{equ:moirl:loss}}
                \STATE{Update $\pi_i^g$ w.r.t. $\tilde{r}_i^g$ using \eqref{equ:gen:loss}}
            \ENDWHILE
            \STATE{Collect dataset $\Dataset_i^g$ by executing $\pi_i^g$}
            \STATE{Execute $\Datasets^g \leftarrow \Datasets^g \cup \Dataset_{i}^{g}$ and $\PolicySet \leftarrow \PolicySet \cup {\pi_i^g}$}
        \ENDFOR
    \ENDFOR
    \STATE{\textbf{return} $\PolicySet$}
    \end{algorithmic}
    \caption{Recursive reward distance regularized IRL} 
    \label{algo:PaIRL}
\end{algorithm}
\subsection{Regret Bounds of Reward Distance Regularized Policy} \label{sec:proof}  
We provide an analysis of the regret bounds of a reward distance regularized policy.
Let $\tilde{r}$ be our learned reward function and $\pi^*_r$ be the optimal policy with respect to reward function $r$. Suppose that there exists a (ground truth) multi-objective reward function $r_{\mo} = \bm{\omega}^T \mathbf{r}$ with preference $\bm{\omega} = [\omega_1,\omega_2]$.
With the linearity of $r_{\mo}$, we obtain
\begin{equation} \label{equ:regbound_1}
\begin{aligned}
    R_{r_{\mo}}(\pi_{r_{\mo}}^*) - R_{r_{\mo}}(\pi_{\tilde{r}}^*)
    & = \sum\nolimits_{i=1}^2 {\omega}_i (R_{\tilde{r}_i}(\pi_{r_{\mo}}^*) - R_{\tilde{r}_i}(\pi_{\tilde{r}}^*)) \\
    & \leq  \sum\nolimits_{i=1}^2 {\omega}_i (R_{\tilde{r}_i}(\pi_{\tilde{r}_{i}}^*) - R_{\tilde{r}_i}(\pi_{\tilde{r}}^*)).
\end{aligned}   
\end{equation}
Let $\mathcal{D}$ be the distribution over transitions $S \times A \times S$ used to compute EPIC distance $d_\epsilon$, and $\mathcal{D}_{\pi, t}$ be the distribution over transitions on timestep $t$ induced by policy $\pi$.
Using Theorem A.16 in~\cite{rw:epic}, we derive that for $\alpha \geq 2$, \eqref{equ:regbound_1} is bounded by the sum of individual regret bounds, i.e.,
\begin{equation} \label{equ:regbound_2}
\begin{aligned}
    \sum\nolimits_{i=1}^2 & {\omega}_i (R_{\tilde{r}_i} (\pi_{\tilde{r}_{i}}^*) - R_{\tilde{r}_i}(\pi_{\tilde{r}}^*)) \\
    & \leq \sum\nolimits_{i=1}^2 {16 {\omega}_i \lVert \tilde{r}_i \rVert_2\left(K d_\epsilon (\tilde{r}, \tilde{r}_i) + L\Delta_\alpha(\tilde{r}) \right)}
\end{aligned}   
\end{equation}
where $L$ is a constant, $K = \alpha / (1-\gamma)$, $\Delta_\alpha(\tilde{r}) = \sum_{t=0}^T\gamma^t W_\alpha(\mathcal{D}_{\pi^*_{\tilde{r}}, t}, \mathcal{D})$, and $W_\alpha$ is the relaxed Wasserstein distance~\cite{wasser}.
Consequently, we obtain
\begin{equation} \label{equ:regbound_3}
\begin{aligned}
    R&_{r_{\text{mo}}}  (\pi_{r_{\text{mo}}}^*)- R_{r_{\text{mo}}}(\pi_{\tilde{r}}^*) \\
    & \leq 32K\|r_{\mo}\|_2 \left(\sum_{i=1}^2 \left[ {\omega}_i d_\epsilon (\tilde{r}, \tilde{r}_i) \right] + \frac{L}{K}\Delta_\alpha(\tilde{r}) \right).
\end{aligned}   
\end{equation}

As such, the regret bounds of our learned policy $\pi$ on reward function $\tilde{r}$ are represented by the regularization term based on EPIC along with the differences between the respective distributions of transitions generated by $\pi^*_{\tilde{r}}$ and the distribution $\mathcal{D}$ used to compute EPIC distance. 
This ensures that the regret bounds of $\pi$ can be directly optimized by using~\eqref{equ:moirl:loss}.
In our implementation, instead of directly multiplying the preference $\bm{\omega}$ to the loss function, we reformulate the preference into the target distance to balance the distance better.
The details with proof can be found in Appendix A.2.
\begin{table*}[t]
    \centering 
    \small
    \begin{adjustbox}{width=1.\textwidth}
    \begin{tabular}{l|c|c|cccccc|cc} 
    \toprule
    \textbf{Environment} & \textbf{Metric} & \textbf{Oracle} & \textbf{DiffBC}  & \textbf{BeT} & \textbf{GAIL} & \textbf{AIRL} & \textbf{IQ-Learn} & \textbf{DiffAIL} & \textbf{$\ourmodel$} & \textbf{$\ourmodel$+DU} \\
    \midrule
    \multirow{3}{*}{\textbf{MO-Car}} 
    & HV & $6.06$ & $3.95 \pm 0.26$ & $4.22 \pm 0.06$ & $4.95 \pm 0.16$ & $5.01 \pm 0.13$ 
    & $3.47 \pm 1.20$ & $1.78 \pm 0.08$ & ${5.37} \pm {0.08}$ & $\textbf{5.89} \pm \textbf{0.05}$ \\
    & SP & $0.01$ & $1.49 \pm 0.19$ & $1.07 \pm 0.23$ & $0.89 \pm 0.43$ & $0.67 \pm 0.10$ 
    & $3.43 \pm 2.84$ & $6.95 \pm 0.12$ & ${0.26} \pm {0.05}$ & $\textbf{0.01} \pm \textbf{0.01}$ \\
    & CR & $1.00$ & $0.81 \pm 0.01$ & $0.82 \pm 0.04$ & $0.60 \pm 0.09$ & $0.83 \pm 0.02$ 
    & $0.69 \pm 0.03$ & $0.69 \pm 0.07$ & $\textbf{0.97} \pm \textbf{0.01}$ & $\textbf{0.97} \pm \textbf{0.01}$ \\ 
    \midrule
    
    \multirow{3}{*}{\textbf{MO-Swimmer}} 
    & HV & $5.60$ & $3.56 \pm 0.38$ & $3.86 \pm 0.20$ & $3.61 \pm 0.13$ & $3.83 \pm 0.38$ 
    & $2.97 \pm 0.35$ & $3.52 \pm 1.03$ & ${4.56} \pm {0.04}$ & $\textbf{4.96} \pm \textbf{0.06}$ \\
    & SP & $0.03$ & $1.12 \pm 0.24$ & $0.74 \pm 0.34$ & $1.80 \pm 0.34$ & $2.34 \pm 1.34$ 
    & $2.54 \pm 1.31$ & $4.00 \pm 4.62$ & ${0.17} \pm {0.02}$ & $\textbf{0.01} \pm \textbf{0.01}$ \\
    & CR & $1.00$ & $0.75 \pm 0.01$ & $0.82 \pm 0.05$ & $0.70 \pm 0.09$ & $0.70 \pm 0.01$ 
    & $0.78 \pm 0.03$ & $0.77 \pm 0.07$ & $\textbf{0.98} \pm \textbf{0.01}$  & $0.96 \pm 0.01$ \\ 
    \midrule
    
    \multirow{3}{*}{\textbf{MO-Cheetah}}
    & HV & $5.09$ & $3.97 \pm 0.27$ & $2.22 \pm 0.29$ & $3.75 \pm 0.28$ & $4.25 \pm 0.06$ 
    & $2.82 \pm 0.48$ & $3.86 \pm 0.46$ & ${4.97} \pm {0.13}$ & $\textbf{5.27} \pm \textbf{0.07}$ \\
    & SP & $0.01$ & $0.59 \pm 0.23$ & $1.26 \pm 0.95$ & $1.56 \pm 0.50$ & $0.62 \pm 0.13$ 
    & $6.68 \pm 3.51$ & $1.40 \pm 0.83$ & ${0.11} \pm {0.01}$ & $\textbf{0.01} \pm \textbf{0.00}$ \\
    & CR & $1.00$ & $0.72 \pm 0.02$ & $0.68 \pm 0.04$ & $0.72 \pm 0.03$ & $0.77 \pm 0.03$ 
    & $0.55 \pm 0.04$ & $0.38 \pm 0.10$ & ${0.92} \pm {0.02}$ & $\textbf{0.93} \pm \textbf{0.02}$ \\ 
    \midrule
    
    \multirow{3}{*}{\textbf{MO-Ant}} 
    & HV & $6.30$ & $3.73 \pm 0.23$ & $1.90 \pm 0.09$ & $3.63 \pm 0.22$ & $4.08 \pm 0.29$ 
    & $2.34 \pm 0.28$ & $1.40 \pm 0.01$ & ${4.71} \pm {0.10}$ & $\textbf{4.99} \pm \textbf{0.12}$ \\
    & SP & $0.01$ & $0.33 \pm 0.13$ & $0.61 \pm 0.79$ & $0.67 \pm 0.13$ & $0.37 \pm 0.16$ 
    & $0.58 \pm 0.24$ & $15.09 \pm 0.01$ & ${0.06} \pm {0.01}$ & $\textbf{0.01} \pm \textbf{0.00}$ \\ 
    & CR & $1.00$ & $0.87 \pm 0.00$ & $0.73 \pm 0.03$ & $0.86 \pm 0.02$ & $0.91 \pm 0.01$ 
    & $0.69 \pm 0.02$ & $0.19 \pm 0.03$ & $\textbf{0.99} \pm \textbf{0.01}$ & $\textbf{0.99} \pm \textbf{0.01}$ \\ 
    \midrule
    
    \multirow{3}{*}{\textbf{MO-AntXY}} 
    & HV & $6.78$ & $3.76 \pm 0.04$ & $1.54 \pm 0.12$ & $4.18 \pm 0.16$ & $4.49 \pm 0.19$ 
    & $2.02 \pm 0.06$ & $3.66 \pm 0.63$ & ${5.37} \pm {0.09}$ & $\textbf{5.61} \pm \textbf{0.10}$ \\
    & SP & $0.03$ & $0.54 \pm 0.11$ & $12.14 \pm 0.88$ & $0.50 \pm 0.02$ & $0.39 \pm 0.14$ 
    & $0.70 \pm 0.13$ & $1.06 \pm 0.24$ & ${0.07} \pm {0.01}$ & $\textbf{0.01} \pm \textbf{0.00}$ \\ 
    & CR & $1.00$ & $0.80 \pm 0.02$ & $0.42 \pm 0.06$ & $0.76 \pm 0.03$ & $0.77 \pm 0.04$ 
    & $0.68 \pm 0.02$ & $0.52 \pm 0.16$  & ${0.95} \pm {0.02}$ & $\textbf{0.98} \pm \textbf{0.00}$ \\ 
    \midrule

    \multirow{2}{*}{\shortstack[l]{\textbf{MO-Car}*\\(3-obj)}}
    & HV & $2.89$ & $0.86 \pm 0.12$ & $1.54 \pm 0.00$ & $1.63 \pm 0.02$ & $1.88 \pm 0.09$ 
    & $0.80 \pm 0.02$ & $1.16 \pm 0.07$ & ${1.87} \pm {0.04}$ & $\textbf{2.79} \pm \textbf{0.02}$ \\
    & SP & $0.01$ & $0.46 \pm 0.12$ & $0.06 \pm 0.00$ & $0.08 \pm 0.01$ & $0.90 \pm 0.74$ 
    & $0.53 \pm 0.03$ & $0.19 \pm 0.02$ & $0.08 \pm 0.00$ & $\textbf{0.01} \pm \textbf{0.00}$ \\ 
    \bottomrule
    \end{tabular}
    \end{adjustbox}
\caption{Performance of Pareto set generation: regarding evaluation metrics, the higher HV, the higher the performance; the lower SP, the higher the performance; the higher CR, the higher the performance. For the baselines and $\ourmodel$, we evaluate with 3 random seeds.
}  
\label{tbl:main1}
\end{table*}
%

\subsection{Preference-conditioned Diffusion Model}
To further enhance the Pareto policy set $\PolicySet$ obtained in the previous section, we leverage diffusion models~\cite{du:ddpm,du:cfdg}, interpolating and extrapolating policies via distillation.
We first systematically annotate $\PolicySet$ with preferences $\bm{\omega} \in \Omega$ in an ascending order. We then train a diffusion-based policy model, which is conditioned on these preferences; i.e.,
\begin{equation}
    {\pi}_{\text{u}}(a|s,\omega) = \mathcal{N}(a^K;0,I) \prod\nolimits_{k=1}^{K} \hat{\pi}_{\text{u}}(a^{k-1}|a^{k},k,s,\omega)
\end{equation}
where superscripts $k \sim [1,K]$ denote the denoising timestep, $a^0 (=a)$ is the original action, and $a^{k-1}$ is a marginally denoised version of $a^k$.
The diffusion model is designed to predict the noise from a noisy input $a^k=\sqrt{\bar{\alpha}^k}a + \sqrt{1-\bar{\alpha}^k}\eta$ with a variance schedule parameter $\bar{\alpha}^k$ and $\eta \sim \mathcal{N}(0,I)$, i.e.
\begin{equation}
    \min \expectation_{(s,a) \sim \{\Datasets^{g}\}_{g=1}^{G}, k\sim[1,K]}[||\hat{\pi}_{\text{u}}(a^k,k,s,\omega) - \eta||_2^2]
\end{equation}
where $\{\Datasets^{g}\}_{g=1}^{G}$ is the entire datasets collected by the policies in $\PolicySet$.
Furthermore, we represent the model as a combination of preference-conditioned and unconditioned policies,
\begin{equation}
\begin{aligned}
    \hat{\pi}_{\text{u}} & (a^k,k,s,\omega) := \\
    & (1-\delta) \hat{\pi}_{\text{cond.}}(a^{k},k,s,\omega) + \delta \hat{\pi}_{\text{uncond.}}(a^k,k,s)
\end{aligned}
\label{equ:prefcondpolicy}
\end{equation}
where $\delta$ is a guidance weight.
The unconditioned policy encompasses general knowledge across the approximated Pareto policies, while the conditioned one guides the action according to the specific preference.

During sampling, the policy starts from a random noise and iteratively denoises it to obtain the executable action,
\begin{equation}
    a^{k-1} = \frac{1}{\sqrt{\alpha^k}}\left(a^k - \frac{1-\alpha^k}{\sqrt{1-\bar{\alpha}^k}}\hat{\pi}_{\text{u}}(a^k,k,s,\omega)\right) + \sigma^k\eta
\end{equation}
where $\alpha^k$ and $\sigma^k$ are variance schedule parameters.
The diffusion model $\hat{\pi}_{\text{u}}$ allows for efficient resource utilization at runtime with a single policy network, and is capable of rendering robust performance for \textit{unseen} preferences in a zero-shot manner.
Consequently, it enhances the Pareto policy set in terms of Pareto front density 
, as illustrated in Figure~\ref{fig:architecture} (\romannumeral 2).
%

\section{Evaluation} \label{section:eval}
\subsection{Experiment Settings}
\noindent\textbf{Environments.} For evaluation, we use (\romannumeral 1) a multi-objective car environment (MO-Car), and several multi-objective variants of MuJoCo environments used in the MORL literature~\cite{morl:pgmorl,morl:ppa} including (\romannumeral 2) MO-Swimmer, (\romannumeral 3) MO-Cheetah, (\romannumeral 4) MO-Ant, and (\romannumeral 5) MO-AntXY. 
For tradeoff objectives, the forward speed and the energy efficiency are used in (\romannumeral 2)-(\romannumeral 4), and the x-axis speed and the y-axis speed are used in (\romannumeral 5).
In these environments, similar to conventional IRL settings, reward signals are not used for training; they are used solely for evaluation.

\noindent\textbf{Baselines.} For comparison, we implement following imitation learning algorithms: 
1) \textbf{DiffBC}~\cite{il:dubc}, an imitation learning method that uses a diffusion model for the policy,
2) \textbf{BeT}~\cite{il:bet}, an imitation learning method that integrates action discretization into the transformer architecture, 
3) \textbf{GAIL}~\cite{il:gail}, an imitation learning method that imitates expert dataset via the generative adversarial framework,
4) \textbf{AIRL}~\cite{irl:airl}, an IRL method that induces both the reward function and policy,
5) \textbf{IQ-Learn}~\cite{irl:iqlearn}, an IRL method that learns a q-function to represent both the reward function and policy,
6) \textbf{DiffAIL}~\cite{irl:diffail}, an IRL method that incorporates the diffusion loss to the discriminator's objective.
To cover a wide range of different preferences, these baselines are conducted multiple times on differently augmented datasets, where each is a mixed dataset that integrates given datasets in the same ratio to a specific preference. 
We also include MORL~\cite{morl:pgmorl} that 
uses explicit rewards from the environment, unlike IRL settings. It serves as \textbf{Oracle} (the upper bound of performance) in the comparison.

%

\noindent\textbf{Metrics.} For evaluation, we use several multi-objective metrics~\cite{morl:moq,morl:pgmorl}.
\begin{itemize}
    \item Hypervolume metric (HV) represents the quality in the cumulative returns of a Pareto policy set. Let $\mathcal{F}$ be the Pareto frontier obtained from an approximated Pareto policy set for $m$ objectives and $\bm{R}_0 \in \mathbb{R}^m$ be a reference point for each objective. Then, $\text{HV} = \int \mathds{1}_{H(\mathcal{F})}(z)dz$ where $H(\mathcal{F})=\{z \in \mathbb{R}^m \mid \exists \bm{R} \in \mathcal{F} : \bm{R_0} \leq z \leq \bm{R}\}$.
    \item Sparsity metric (SP) represents the density in the average return distance of the Pareto frontier. Let $\mathcal{F}_j(i)$ be the $i$-th value in a sorted list for the $j$-th objective. 
    Then, $\text{SP} = \frac{1}{|\mathcal{F}| - 1} \sum_{j=1}^{m} \sum_{i=1}^{|\mathcal{F}|} (\mathcal{F}_j(i) - \mathcal{F}_j(i+1))^2$.
\end{itemize}
We also use a new metric designed for Pareto IRL. 
\begin{itemize}
    \item Coherence metric (CR) represents the monotonic improvement property of approximated policy set $\PolicySet = \{\pi_i\}_{i\leq N}$ generated by two expert datasets. Let policy list $(\pi_1, ..., \pi_N)$ be sorted in ascending order by the expected return of the policies with respect to reward function $r_1$, Then, $\text{CR} = \frac{2}{N (N-1)} \sum_{i=1}^{N} \sum_{j=i}^{N} \mathds{1}_{h(i,j)}$ where $h(i,j) \!\! = \!\! R_{r_1}(\pi_i) \!\! \leq \!\! R_{r_1}(\pi_j) \ \text{and} \ R_{r_2}(\pi_i) \!\! \geq \!\! R_{r_2}(\pi_j)$. 
\end{itemize}
For HV and CR, higher is better, but for SP, lower is better. 

\subsection{Performance of Pareto Set Generation}
Table~\ref{tbl:main1} compares the performance in the evaluation metrics (HV, SP, CR) achieved by our framework ($\ourmodel$, $\ourmodel$+DU) and other baselines (DiffBC, BeT, GAIL, AIRL, IQ-Learn, DiffAIL).
$\ourmodel$ is trained with the recursive reward distance regularized IRL, and $\ourmodel$+DU is enhanced through the distillation. 
%
For the baselines, the size of a preference set (with different weights) is given equally to the number of policies derived via $\ourmodel$.
%
When calculating HV and SP, we exclude the out-of-order policies obtained from an algorithm with respect to preferences.
As shown, our $\ourmodel$ and $\ourmodel$+DU consistently yield the best performance for all environments, outperforming the most competitive baseline AIRL by $15.6\% \sim 23.7\%$ higher HV, $80.4\% \sim 98.2\%$ lower SP, and $21.7\% \sim 22.2\%$ higher CR on average.
Furthermore, we observe an average HV gap of $9.8\%$ between $\ourmodel$+DU and Oracle that uses the ground truth reward signals.
This gap is expected, as existing IRL algorithms are also known to experience a performance drop compared to RL algorithms that directly use reward signals~\cite{irl:airl}.
For the baselines, such performance degradation is more significant, showing an average drop of $26.9\%$ in HV between AIRL and Oracle.
$\ourmodel$+DU improves the performance in HV over $\ourmodel$ by $7.0\%$ on average, showing the distilled diffusion model achieves robustness on unseen preferences.

To verify the performance of $\ourmodel$ for three objectives case, we extend MO-Car to MO-Car* where the tradeoff objectives are the velocities on three different directions.
Our $\ourmodel$ and $\ourmodel$+DU show superiority in terms of HV, but sometimes show slightly lower performance in SP. I
t is because the baselines tend to shrinks towards the low-performance region, thus yielding lower SP.
As CR is defined only for two objectives cases, CR for MO-Car* is not reported. 
The generalization of reward distance regularization for three or more objectives is discussed in Appendix A.4.

In this experiment, the baselines exhibit relatively low performance due to their primarily concentration on imitating the datasets, posing a challenge in generating policies that go beyond the limited datasets.
Specifically, as DiffBC and BeT are designed to handle datasets with multiple modalities, they do not necessarily lead to the generation of novel actions.
Meanwhile, the IRL baselines demonstrate relatively better performance, as they involve environment interactions.
However, imitating from a merged dataset with specific ratio tends to converge towards the mean of existing actions, thus  leading to sub-optimal performance.

\subsection{Analysis}
\noindent\textbf{Pareto Visualization.} Figure~\ref{fig:ana:pareto} depicts the Pareto policy set by our $\ourmodel$ and $\ourmodel$+DU as well as the baselines (DiffBC, AIRL) for MO-AntXY. 
The baselines often produce the non-optimal solutions, specified by the dots in the low-performance region. 
%
$\ourmodel$+DU produces the most densely spread policies, which lie on the high-performance region.
%
%

\noindent\textbf{Learning Efficiency.} Figure~\ref{fig:ana:curve} depicts the learning curves in HV for MO-AntXY over recursive steps.
For baselines, we intentionally set the number of policies of the baselines equal to the number of policies derived through $\ourmodel$ for each step.
The curves show the superiority of our recursive reward distance regularized IRL in generating the higher quality (HV) Pareto frontier.
Furthermore, the recursive learning scheme significantly reduces the training time, requiring only $13\% \sim 25\%$ of training timesteps compared to the IRL baselines.
This is because $\ourmodel$ explores adjacent policies progressively by making explicit use of the previously derived policies to fork another regularized IRL procedure.
\begin{figure}[h]
    \centering
    \subfigure[Pareto Visualization]{
        \centering
        \includegraphics[width=0.47\linewidth]{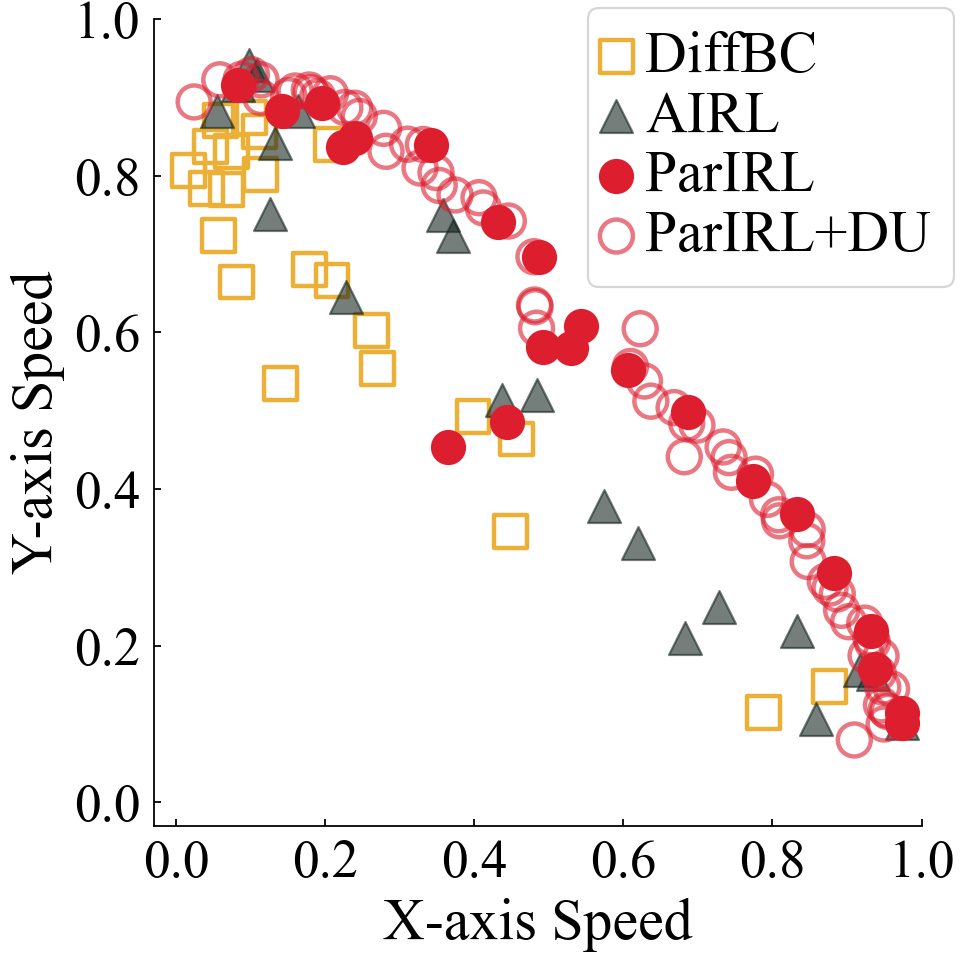}
        \label{fig:ana:pareto}
    }
    \subfigure[Learning Curve]{
        \centering
        \includegraphics[width=0.47\linewidth]{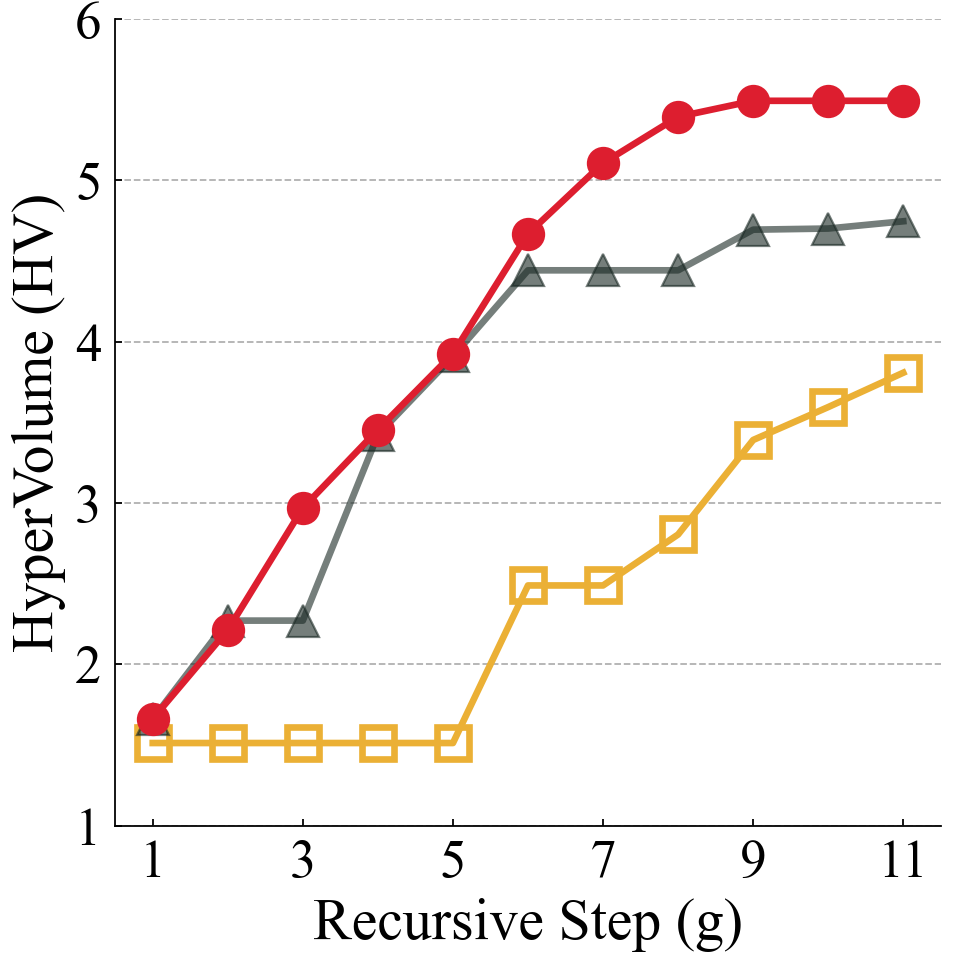}
        \label{fig:ana:curve}
    }
    \caption{Pareto policy set $\PolicySet$ visualization and learning curve}
    \label{fig:ana}
\end{figure}
%

\noindent\textbf{Ablation Studies.} Table~\ref{tbl:abl:1} provides an ablation study of $\ourmodel$ with respect to the reward distance metrics and recursive learning scheme.
For this, we implement $\ourmodel$/MSE and $\ourmodel$/PSD, which use mean squared error (MSE) and Pearson distance (PSD) for reward distance measures, respectively; we also implement $\ourmodel$/RC which represents $\ourmodel$ without recursive learning scheme.
While MSE tends to compute the exact reward distance and PSD estimates the linear correlation between rewards, EPIC accounts for the reward function distance that is invariant to potential shaping~\cite{rw:epic}, thus making $\ourmodel$ optimize the regret bounds of a policy learned on an inferred reward function.
Moreover, $\ourmodel$/RC degrades compared to $\ourmodel$, clarifying the benefit of our recursive learning scheme.
\begin{table}[h]
    \centering 
    \small
    \begin{adjustbox}{width=\linewidth}
    \begin{tabular}{c|c|cccc}
        \toprule
        \textbf{Env.} & \textbf{Met.} & \textbf{$\ourmodel$/MSE} & \textbf{$\ourmodel$/PSD} & \textbf{$\ourmodel$/RC} & \textbf{$\ourmodel$} \\
        \midrule
        \multirowcell{2}{1}
        & HV & $3.37 \pm 0.08$ & $4.02 \pm 0.23$ & $4.17 \pm 0.14$ & $\textbf{4.97} \pm \textbf{0.13}$ \\ 
        & SP & $1.93 \pm 0.31$ & $0.55 \pm 0.14$ & $0.47 \pm 0.18$ & $\textbf{0.11} \pm \textbf{0.01}$ \\
        \midrule
        \multirowcell{2}{2} 
        & HV & $2.28 \pm 0.07$ & $4.96 \pm 0.11$ & $4.10 \pm 0.23$ & $\textbf{5.37} \pm \textbf{0.09}$ \\
        & SP & $2.25 \pm 0.23$ & $0.29 \pm 0.04$ & $0.58 \pm 0.23$ & $\textbf{0.07} \pm \textbf{0.01}$ \\
        \bottomrule
    \end{tabular}
    \end{adjustbox}
    \caption{Performance w.r.t reward distance metrics: 1 and 2 represents MO-Cheetah and MO-AntXY, respectively.}
    \label{tbl:abl:1}
\end{table}
Table~\ref{tbl:abl:2} shows the effect of our preference-conditioned diffusion model.
$\ourmodel$+BC denotes distillation using the naive BC algorithm. We test $\ourmodel$+DU with varying guidance weights $\delta$ in~\eqref{equ:prefcondpolicy}, ranging from $0.0$ to $1.8$.
The results indicate that $\ourmodel$+DU improves by $6.42\%$ at average over $\ourmodel$+BC. Employing both unconditioned and conditioned policies ($\delta > 0$) contributes to improved performance.
\begin{table}[h]
    \centering 
    \small
    \begin{adjustbox}{width=\linewidth}
    \begin{tabular}{c|c|c|ccc}
        \toprule
        \textbf{Env.} & \textbf{Met.} & \textbf{$\ourmodel$+BC} & \textbf{$\delta=0.0$} & \textbf{$\delta=1.2$} & \textbf{$\delta=1.8$} \\
        \midrule
        \multirowcell{2}{1}
        & HV & $4.52 \pm 0.46$ & $4.86 \pm 0.09$ & $\textbf{4.96} \pm \textbf{0.06}$ & $4.94 \pm 0.05$ \\ 
        & SP & $0.02 \pm 0.00$ & $0.01 \pm 0.00$ & $0.01 \pm 0.00$ & $0.01 \pm 0.00$ \\
        \midrule
        \multirowcell{2}{2} 
        & HV & $5.48 \pm 0.05$ & $5.54 \pm 0.10$ & $5.61 \pm 0.10$ & $\textbf{5.65} \pm \textbf{0.07}$ \\
        & SP & $0.01 \pm 0.00$ & $0.01 \pm 0.00$ & $0.01 \pm 0.00$ & $0.01 \pm 0.00$ \\
        \bottomrule
    \end{tabular}
    \end{adjustbox}
    \caption{Performance of preference-conditioned diffusion models: 1 and 2 represent MO-Swimmer and MO-AntXY, respectively.}
    \label{tbl:abl:2}
\end{table}
\subsection{Case Study on Autonomous Driving}
To verify the applicability of our framework, we conduct a case study with autonomous driving scenarios in the CARLA simulator~\cite{carla}.
In Figure~\ref{fig:carla}, the comfort mode agent drives slowly without switching lanes, while the sport mode agent accelerates and frequently switches lanes (indicated by dotted arrow) to overtake front vehicles (highlighted by dotted circle) ahead.
Using the distinct datasets collected from these two different driving modes, $\ourmodel$ generates a set of diverse \textit{custom} driving policies.
Specifically, as depicted in the bottom of Figure~\ref{fig:carla}, the closer the custom agent's behavior is to the sport mode, the more it tends to switch lanes (increasing from $0$ to $2$) and to drive at higher speeds with lower energy efficiency. The agent in custom mode-2 balances between the comfort and sport modes well, maintaining the moderate speed and changing lanes once.
\begin{figure}[ht]
    \centering
    \includegraphics[width=\linewidth]{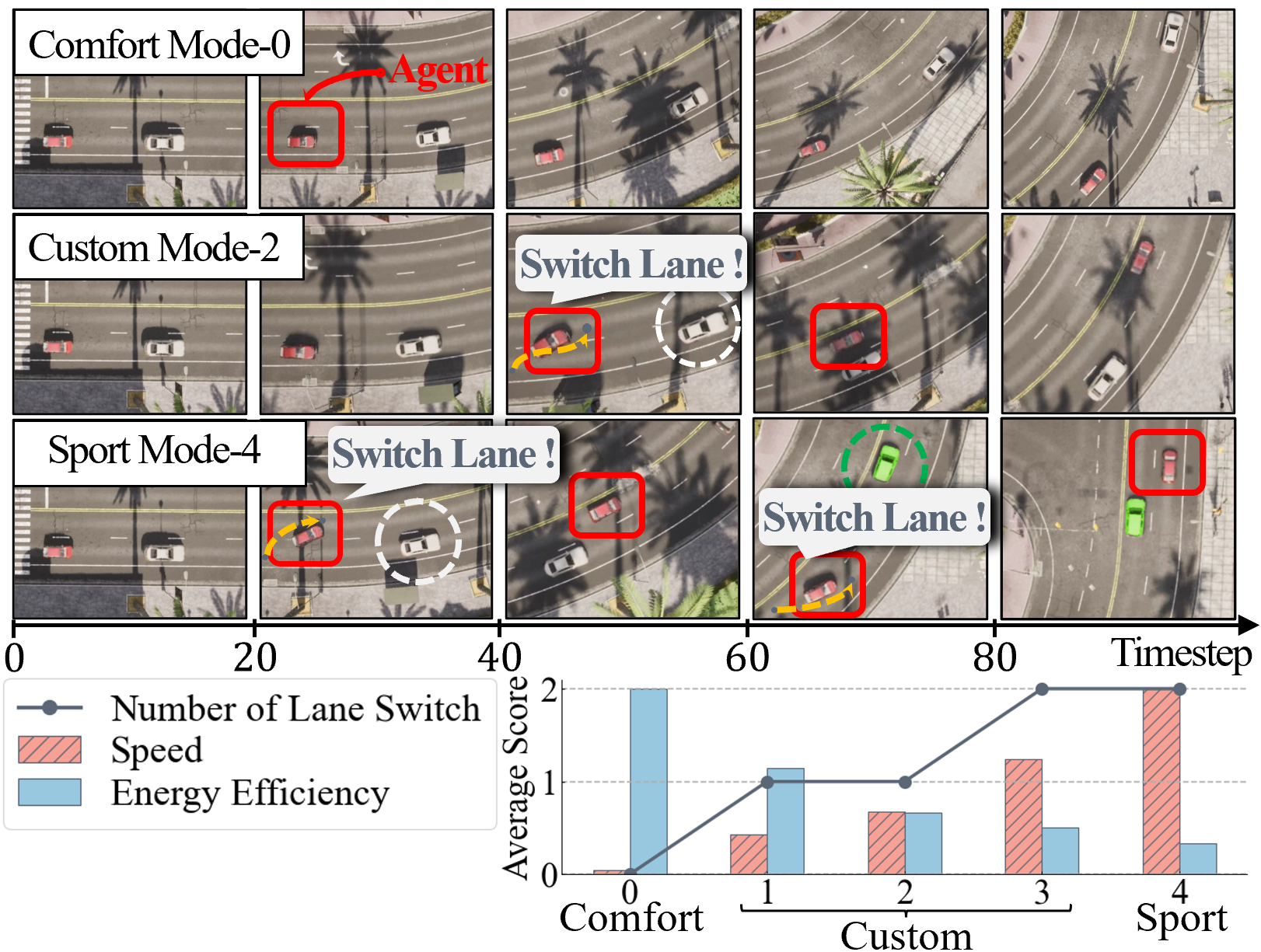}
    \caption{Visualization of agents obtained via $\ourmodel$: the red rectangle denotes learned agent, the dotted lines denote lane changes, the dotted circles denote the vehicles overtaken by our agent.}
    \label{fig:carla}
\end{figure}

\section{Related Work}
\noindent \textbf{Multi-objective RL.}
In the RL literature, several multi-objective optimization methods were introduced, aiming at providing robust approximation of a Pareto policy set.
%
%
\cite{morl:moq,morl:pgmorl} explored Pareto policy set approximation through reward scalarization in online settings, where reward signals are provided. 
%
%
%
%
Recently, \cite{moil:pdt} proposed the Pareto decision transformer in offline settings, requiring a comprehensive dataset that covers all preferences.
These prior works and ours share a similar goal to achieve a tradeoff-aware agent based on Pareto policy set approximation.
However, different from the prior works, our work concentrates on practical situations with the strictly limited datasets and without any rewards from the environment.

\noindent \textbf{Inverse RL.}
To infer a reward function from datasets, IRL has been investigated along with adversarial schemes.
\cite{irl:airl} established the practical implementation of IRL based on the generative adversarial framework; which was further investigated by~\cite{irl:mlirl,irl:iqlearn,irl:diffail}.
Recently, \cite{moirl:mf} introduced a multi-objective reward function recovery method, using a simple discrete grid-world environment.
Contrarily, our $\ourmodel$ targets the approximation of a Pareto policy set. 
Instead of exploring the linear combinations of rewards, $\ourmodel$ employs the reward distance metric, and further, optimizes the performance lower bound of learned policies.

\noindent \textbf{Reward Function Evaluation.}
Reward function evaluation is considered important in the RL literature, but was not fully investigated. 
\cite{rw:epic} first proposed the EPIC by which two reward functions are directly compared without policy optimization, and verified that the policy regret is bounded. This was extended by~\cite{rw:dard} for mitigating erroneous reward evaluation. 
%
%
However, those rarely investigated how to use such metrics for multi-objective learning. 
Our work is the first to conjugate reward function evaluation for Pareto policy set approximation in IRL settings.
\section{Conclusion}
We presented the $\ourmodel$ framework to induce a Pareto policy set from strictly limited datasets in terms of preference diversity.
In $\ourmodel$, the recursive IRL with the reward distance regularization is employed to achieve the Pareto policy set.
The set is then distilled to the preference-conditioned diffusion policy, enabling robust policy adaptation to unseen preferences and resource efficient deployment. Our framework is different from the existing IRL approaches in that they only allow for imitating an individual policy from given datasets.
%
%
\section*{Acknowledgements}
We would like to thank anonymous reviewers for their valuable comments. 
This work was supported by 
Institute of Information \& communications Technology Planning \& Evaluation (IITP) grant funded by the Korea government (MSIT) (No.
2022-0-01045, 
2022-0-00043, 
2021-0-00875, 
2020-0-01821, 
2019-0-00421) 
 and by the National Research Foundation of Korea (NRF) grant funded by the MSIT 
(No. RS-2023-00213118) 
and by Samsung electronics.



\bibliographystyle{named}
\bibliography{ijcai24}

\newpage 
\appendix
\input{appendix}

\end{document}

%% file: appendix.tex
\appendix

\section{Reward Distance Regularization}
In this section, we briefly explain the EPIC distance and provide theoretical analysis on our reward distance regularized loss based on EPIC.
Then, we discuss our motivation for target distance and generalize our regularized loss for more than two objectives cases.

\subsection{Equivalent Policy Invariant Comparison (EPIC)}
EPIC~\cite{rw:epic} is computed using the Pearson distance between two canonically shaped reward functions for independent random variables $\StateSet$, $\StateSet'$ sampled from state distribution $\mathcal{D}_\StateSet$ and $\ActionSet$ sampled from action distribution  $\mathcal{D}_\ActionSet$. Then, EPIC distance $d_\epsilon$ between two reward functions $r$ and $r'$ is calculated by
\begin{equation} \label{equ:epic}
\begin{aligned}
    d_{\epsilon} (r,r') = d_\rho ({C}_{\mathcal{D}_S,\mathcal{D}_A}(r) & (S,A,S'), \\
    & {C}_{\mathcal{D}_S, \mathcal{D}_A}(r')(S,A,S'))
\end{aligned}
\end{equation}
where $d_\rho(X,Y) = \sqrt{1 - \rho(X,Y)} / \sqrt{2}$ and $\rho(X,Y)$ is the Pearson correlation between random variables $X$ and $Y$. The canonicalized reward function $C$ is defined as
\begin{equation}
\label{equ:epic-canon}
\begin{aligned}
   {C}&_{\mathcal{D}_S, \mathcal{D}_A} (r)(s,a,s') = r(s,a,s') \\
   & + \mathbb{E} [ \gamma r(s',A, S') - r(s,A,S') - \gamma r(S,A,S') ]. 
\end{aligned}
\end{equation}

\subsection{Regret Bounds of Reward Distance Regularization Based on EPIC}
In this section, we present the proof of regret bounds for the policy learned through our reward distance regularization based on EPIC. 
We start by Lemma~\ref{lemma:1}, where we show the relaxed Wasserstien distance $W_\alpha$ equals to $0$ between a distribution $\mathcal{D}_i$ and $\frac{1}{m}\sum_{i=1}^m \mathcal{D}_i$.
Next, in Theorem~\ref{theorem:1}, we prove the regret bounds in terms of the EPIC distance between the reward functions and $W_\alpha$.

\begin{lemma} \label{lemma:1} 
    Let $\mathcal{D}_1, ..., \mathcal{D}_m$ be arbitrary distributions over transitions $S \times A \times S$. For $\alpha \geq m$ and $i \in \{1,...,m\}$, 
    \begin{equation}
        W_\alpha(\mathcal{D}_i, (\mathcal{D}_1 + ... + \mathcal{D}_m) / m) = 0
    \end{equation}
    where $W_\alpha$ is relaxed Wasserstein distance (Definition A.13 in~\cite{rw:epic}).
\end{lemma}
\begin{proof}
    For simplicity, We denote $\mathcal{D} = (\mathcal{D}_1 + ... + \mathcal{D}_m) / m$. By the definition of relaxed Wasserstein distance,
    \begin{equation}
     W_\alpha(\mathcal{D}_i, \mathcal{D})  = \inf_{p\in\Gamma_\alpha(\mathcal{D}_i, \mathcal{D})} \int_{S\times S} \| x-y \|dp(x,y)
    \end{equation}
    where $\Gamma_\alpha(\mathcal{D}_i, \mathcal{D})$ is a set of probability measures on $S \times S$ satisfying
    \begin{equation}
            \int_S p(x, y)dy = \mathcal{D}_i(x),\ \int_S p(x, y)dx \leq \alpha \mathcal{D}(y)
    \end{equation}
    for all $x, y \in S$.
    Let the set $S_D = \{(x, x) | \, x \in S\}$ be a diagonal set on $S \times S$.
    For function $f: S \rightarrow S_D$ such as $f(x) = (x, x)$, the Borel probability measure $\mu:S_D \rightarrow \mathbb{R}$ is defined as $\mathcal{D}_i \circ f^{-1}$. 
    Furthermore, for all Borel sets $X \in S \times S$, the Borel probability measure $p$ on $S \times S$ is defined as $p(X) = \mu(X \cap S_D)$~\cite{wasser}. Then, 
    \begin{equation}
            \int_S p(x, y)dy = \mathcal{D}_i(x),\ \int_S p(x, y)dx = \mathcal{D}_i(y)
    \end{equation}
    hold for all $x, y \in S$. Since $\mathcal{D}_i$ is non-negative and finite, for all $i \in \{1,...m\}$, we obtain
    \begin{equation}
        \int_S p(x, y)dx = \mathcal{D}_i(y) \leq m \cdot \mathcal{D}(y).
    \end{equation}
    Thus, the relaxed Wasserstein distance between $\mathcal{D}_i$ and $\mathcal{D}$ is equal to 
    \begin{equation}
    \begin{aligned}
        & \int_{S\times S} \| x-y \|dp(x,y) \\
        & = \int_{S_D} \| x - y \| dp(x, y) + \int_{S_D^c} \| x - y \| dp(x, y) = 0
    \end{aligned}
    \end{equation}
    where $S_D^c = S \times S \setminus S_D$.
\end{proof}
\begin{theorem} \label{theorem:1}
    Let $\mathcal{D}$ be the distribution over transitions $S \times A \times S$ that is used to compute EPIC distance $d_\epsilon$, and let  $\mathcal{D}_{\pi, t}$ be the distribution over the transitions on timestep $t$ induced by policy $\pi$. 
    Let $\tilde{r}$ be our learned reward function, and let $\pi^*_r$ be the optimal policy with respect to reward function $r$. 
    Suppose that there exists a (ground truth) multi-objective function $r_{\mo} = \bm{\omega}^T \mathbf{r}$ with preference $\bm{\omega} = [\omega_1, ..., \omega_m]$. 
    For $\alpha \geq m$, the regret bounds of $\pi^*_{\tilde{r}}$ at most correspond to 
    \begin{equation}
        \begin{aligned}
            R&_{r_{\text{mo}}}(\pi_{r_{\text{mo}}}^*)- R_{r_{\text{mo}}}(\pi_{\tilde{r}}^*) \\
            & \leq 16mK\|r_{\mo}\|_2 \left(\sum_{i=1}^n \left[ \omega_i d_\epsilon (\tilde{r}, \tilde{r}_i) \right] + \frac{L}{K}\Delta_\alpha(\tilde{r}) \right)
        \end{aligned}
    \end{equation}
    where $\Delta_\alpha(\tilde{r}) = \sum_{t=0}^T\gamma^t W_\alpha(\mathcal{D}_{\pi^*_{\tilde{r}}, t}, \mathcal{D})$ and $K = \alpha / (1-\gamma)$.
\end{theorem}
\begin{proof}
    According to Theorem A.16 in~\cite{rw:epic}, for any $\alpha \geq 1$, the regret bounds of a policy $\pi_{r_i}^*$ for reward function $r_j$ are calculated by
    \begin{equation}
        \begin{aligned}
            R&_{r_j}(\pi_{r_j}^*) - R_{r_j}(\pi_{r_i}^*) \\
            & \leq 16 \lVert r_j \rVert_2\left(\frac{\alpha}{1-\gamma} d_\epsilon (r_i, r_j) + L\sum_{t=0}^{T}\gamma^t B_\alpha (t)\right)
        \end{aligned}
        \label{equ:regbound}
    \end{equation}
    where $R_{r}(\pi)$ denotes returns of policy $\pi$ on reward function $r$, and reward functions $r_i, r_j$ are L-lipschitz continuous on the $L_1$ norm.
    With the linearity of $r_{\mo}$, we obtain
    \begin{equation}
        \begin{aligned}
            R_{r_{\mo}}(\pi_{r_{\mo}}^*) & - R_{r_{\mo}}(\pi_{\tilde{r}}^*) \\
            & = \sum_{i=1}^m \omega_i (R_{\tilde{r}_i}(\pi_{r_{\mo}}^*) - R_{\tilde{r}_i}(\pi_{\tilde{r}}^*)) \\
            & \leq  \sum_{i=1}^m \omega_i (R_{\tilde{r}_i}(\pi_{\tilde{r}_{i}}^*) - R_{\tilde{r}_i}(\pi_{\tilde{r}}^*)).
        \end{aligned}
        \label{equ:regbound_2}
    \end{equation}
    By~\eqref{equ:regbound}, then, the last term in~\eqref{equ:regbound_2} is bounded by the sum of individual regret bounds, i.e., 
    \begin{equation}
        \begin{aligned}
            \sum_{i=1}^m & \omega_i (R_{\tilde{r}_i}(\pi_{\tilde{r}_{i}}^*) - R_{\tilde{r}_i}(\pi_{\tilde{r}}^*)) \\
            & \leq \sum_{i=1}^m{16 \omega_i \lVert \tilde{r}_i \rVert_2\left(K d_\epsilon (\tilde{r}, \tilde{r}_i) + L\sum_{t=0}^{T}\gamma^t B_\alpha (t) \right)}
        \end{aligned}
    \end{equation}
    where $K = \alpha / (1-\gamma)$.
    Since $\mathcal{D}$ is equivalent to the distribution over transitions induced by $\pi^*_{\tilde{r}_1}, ..., \pi^*_{\tilde{r}_m}$ in our sampling procedure, $B_\alpha$ can be simplified in 
    \begin{equation}
        \begin{aligned}
            B_\alpha(t) & = \max_{\pi\in\{\pi_{\tilde{r}_i}^*, \pi_{\tilde{r}}^*\}} W_\alpha(\mathcal{D}_{\pi, t}, \mathcal{D}) \\
            & = W_\alpha(\mathcal{D}_{\pi^*_{\tilde{r}}, t}, \mathcal{D})
        \end{aligned}
    \end{equation}
    for $\alpha \geq m$ by Lemma~\ref{lemma:1}. 
    For simplicity, we use the episodic cumulative Wasserstein distance $\Delta_\alpha(\tilde{r}) = \sum_{t=0}^T \gamma^t B_\alpha(t)$.
    In practice, reward function $\tilde{r}_i$ is bounded by some constant. Consequently, we obtain
    \begin{equation}
        \begin{aligned}
            R&_{r_{\mo}}(\pi_{r_{\mo}}^*)- R_{r_{\mo}}(\pi_{\tilde{r}}^*) \\
            & \leq 16mK\| r_{\mo} \| \left(\sum_{i=1}^n \left[ \omega_i d_\epsilon (\tilde{r}, \tilde{r}_i) \right] + \frac{L}{K}\Delta_\alpha(\tilde{r}) \right).
        \end{aligned}
    \end{equation}
    This ensures that the regret bounds of $\pi^*_{\tilde{r}}$ with respect to $r_{\mo}$ can be directly optimized by using the loss (11) in the main manuscript.
\end{proof}

\subsection{Motivation for Target Distance}
Here we discuss our motivation for the target distance $\bm{\epsilon}^g$ mentioned in (6)-(8) of the main manuscript.
A straightforward approach for incorporating the reward distance regularized loss is to use the weighted sum loss of preference weight and reward distances. However, we observe that the weighted sum loss frequently leads to unstable learning when targeting to balance between the reward distances.
Thus, we take a different approach, using the target distance, in a way that the target distance is used as the target for the reward distances.

Specifically, as the triangle inequality of reward distance metrics, the sum of the target distance cannot exceed the distance between the reward functions derived from the previous step.
Thus, by setting the sum of the target distance as defined in (8) in the main manuscript and using the L2 loss between the target and reward distances, we are able to stabilize the learning procedure in $\ourmodel$.
Furthermore, we deliberately set one of the target distances (specifically, $\epsilon^g_{i,i}$) to be as small as possible. 
This allows for gradual interpolation between adjacently learned policies. 
Table~\ref{tbl:refpref} demonstrates the effectiveness of the target distance, showing $8.85\%$ gain in HV over the naive approach that directly uses the reward distance metric in the form of the \textit{weighted sum} loss.
\begin{table}[h]
    \centering
    \small
    \begin{tabular}{lcc}
    \toprule
    \textbf{Environment} & \textbf{Weighted sum}   & \textbf{$\ourmodel$} \\
    \midrule 
    \textbf{MO-Car} & $5.10 \pm 0.04$ & $5.37 \pm 0.08$ \\
    \textbf{MO-Cheetah} & $4.34 \pm 0.13$ & $4.97 \pm 0.13$ \\
    \bottomrule
    \end{tabular}
    \caption{Performance in HV w.r.t preference weights}
    \label{tbl:refpref}
\end{table}

\subsection{Generalization of Reward Distance Regularization}
In this section, we extend our reward distance regularization to accommodate general cases involving more than two objectives ($M \geq 3)$.
%
%
Similar to the two objective case, we consider the triangle inequality of the reward distance metric between the learning reward function $\tilde{r}^g_i$ and any two arbitrary reward functions $\tilde{r}^{g-1}_{k}, \tilde{r}^{g-1}_{l} \in \tilde{\mathbf{r}}^{g-1}$ derived in the previous step,
\begin{equation} \label{equ:ntriangle}
    d(\tilde{r}^g_i, \tilde{r}^{g-1}_k) + d(\tilde{r}^g_i, \tilde{r}^{g-1}_l) \geq d(\tilde{r}^{g-1}_k, \tilde{r}^{g-1}_l).
\end{equation}
By leveraging this inequality, we limit the sum of target distances as
\begin{equation} \label{equ:moirl:trdsum}
     \hat{\epsilon}_i^g = \sum\nolimits_{j=1}^{M} \epsilon_{i,j}^{g} =
     \frac{1}{M-1} \sum\nolimits_{k=1}^{M} \sum\nolimits_{l=1}^{M} d(\tilde{r}_{k}^{g-1},\tilde{r}_{l}^{g-1}).
\end{equation}
Note that $M$ is number of objectives, which is equivalent to the number of given expert datasets.
For example, when $M=3$, let the three reward functions derived in the previous step be $\tilde{r}^{g-1}_1,\tilde{r}^{g-1}_2$ and $\tilde{r}^{g-1}_3$.
Using the triangle inequality in~\eqref{equ:ntriangle}, we establish the following set of inequalities for the newly derived one $\tilde{r}^g_i$ 
\begin{equation}
\begin{aligned} \label{equ:3triangle}
    & d(\tilde{r}^g_i, \tilde{r}^{g-1}_1) + d(\tilde{r}^g_i, \tilde{r}^{g-1}_2) \geq d(\tilde{r}^{g-1}_1, \tilde{r}^{g-1}_2) \\
    & d(\tilde{r}^g_i, \tilde{r}^{g-1}_2) + d(\tilde{r}^g_i, \tilde{r}^{g-1}_3) \geq d(\tilde{r}^{g-1}_2, \tilde{r}^{g-1}_3) \\
    & d(\tilde{r}^g_i, \tilde{r}^{g-1}_3) + d(\tilde{r}^g_i, \tilde{r}^{g-1}_1) \geq d(\tilde{r}^{g-1}_3, \tilde{r}^{g-1}_1).
\end{aligned}
\end{equation}
Then, by combining all the inequalities, we obtain
\begin{equation} \label{equ:3trdsum}
\begin{aligned}
    d&(\tilde{r}^g_i, \tilde{r}^{g-1}_1) + d(\tilde{r}^g_i, \tilde{r}^{g-1}_2) + d(\tilde{r}^g_i, \tilde{r}^{g-1}_3) \\
    & \geq \frac{1}{2} \left( d(\tilde{r}^{g-1}_1, \tilde{r}^{g-1}_2) +(\tilde{r}^{g-1}_2, \tilde{r}^{g-1}_3) + d(\tilde{r}^{g-1}_3, \tilde{r}^{g-1}_1) \right).
\end{aligned}
\end{equation}
Finally, the right-hand side of~\eqref{equ:3trdsum} above is set to be the sum of the target distances for $M=3$.
We evaluate the extensibility of $\ourmodel$ for more than two objective in MO-Car* environment, where the result is shown in the Table 1 of the main manuscript.

\section{Benchmark Environments}
In this section, we show the details about our multi-objective environments used for evaluation.
\subsection{MO-Car} \label{env:carcontrol}
MO-Car is a simple 1D environment, where the agent controls the car with an acceleration $a\in[-1,1]$. We configure the two objectives as forward speed and energy efficiency,
\begin{equation}
    \begin{cases}
    & r_1 = 0.05 \times v \\
    & r_2 = 0.3 - 0.15 a^2
    \end{cases}
\end{equation}
where $v$ is the speed.
%
%
%
\subsection{MO-Swimmer}
MO-Swimmer is a multi-objective variant of the MuJoCo~\cite{mujoco} Swimmer environment, where an agent moves forward by applying torques on two rotors. We configure the two objectives as forward speed and energy efficiency,
\begin{equation}
    \begin{cases}
    & r_1 = v_x \\
    & r_2 = 0.3 - 0.15 \sum\nolimits_i a_i^2
    \end{cases}
\end{equation}
where $v_x$ is the speed in $x$ direction, and $a_i$ is the action applied to each rotors.
%
%
\subsection{MO-Cheetah}
MO-Cheetah is a multi-objective variant of the MuJoCo HalfCheetah environment, where an agent moves forward by applying torques on 6 distinct joints of front and back legs. We configure the two objectives as forward speed and energy efficiency,
\begin{equation}
    \begin{cases}
    & r_1 = \text{min}(v_x,4) \\
    & r_2 = 4 - \sum\nolimits_i a_i^2
    \end{cases}
\end{equation}
where $v_x$ is the speed in $x$ direction, and $a_i$ is the action applied to each joints.
%

\subsection{MO-Ant}
MO-Ant is a multi-objective variant of the MuJoCo Ant environment, where an agent moves forward by applying torques on 8 distinct rotors of 4 legs. We configure the two objectives as forward speed and energy efficiency,
\begin{equation}
    \begin{cases}
    & r_1 = v_x \\
    & r_2 = 4 - \sum\nolimits_i a_i^2
    \end{cases}
\end{equation}
where $v_x$ is the speed in $x$ direction, and $a_i$ is the action applied to each rotors.
%
%
\subsubsection{MO-AntXY} \label{env:antxy}
MO-Ant is another multi-objective variant of the MuJoCo Ant environment. We configure the two objectives as x-axis speed and y-axis speed,
\begin{equation}
    \begin{cases}
    & r_1 = v_x + C \\
    & r_2 = v_y + C
    \end{cases}
\end{equation}
where $C=2\sum\nolimits_i a_i^2$ is the energy efficiency, $v_x$ is the speed in $x$ direction, $v_y$ is the speed in $y$ direction, and $a_i$ is the action applied to each rotors.
%
%
\subsection{MO-Car*}
MO-Car* is a variant of MO-Car in Section~\ref{env:carcontrol}, where the agent moves in three directions. We configure the three objectives as x-axis speed, y-axis speed and z-axis speed, 
\begin{equation}
    \begin{cases}
    & r_1 = v_x \\
    & r_2 = v_y \\
    & r_3 = v_z
    \end{cases}
\end{equation}
where $v_x$ is the speed in $x$ direction, $v_y$ is the speed in $y$ direction, and $v_z$ is the speed in $z$ direction.
%

\subsection{Case Study on CARLA}
For case study, we use CARLA~\cite{carla}, where an agent drives along the road with obstacles. 
The agent receives lidar information and an image of size $84 \times 84 \times 3$ as an input. Particularly, the image is processed by image encoder pre-trained with images obtained in CARLA.
Figure~\ref{fig:app:carla_map} visualizes the driving map used in CARLA and Figure~\ref{fig:app:carla_input} shows an example image input.
We configure two objectives as forward speed and enrgy consumption,
\begin{equation}
    \begin{cases}
    & r_1 = v \\
    & r_2 = 1 - a^2 \\
    \end{cases}
\end{equation}
where $v_x$ is the speed. 
%
\begin{figure}[h]
    \centering
        \subfigure[Map of CARLA]{
            \includegraphics[width=0.46\linewidth]{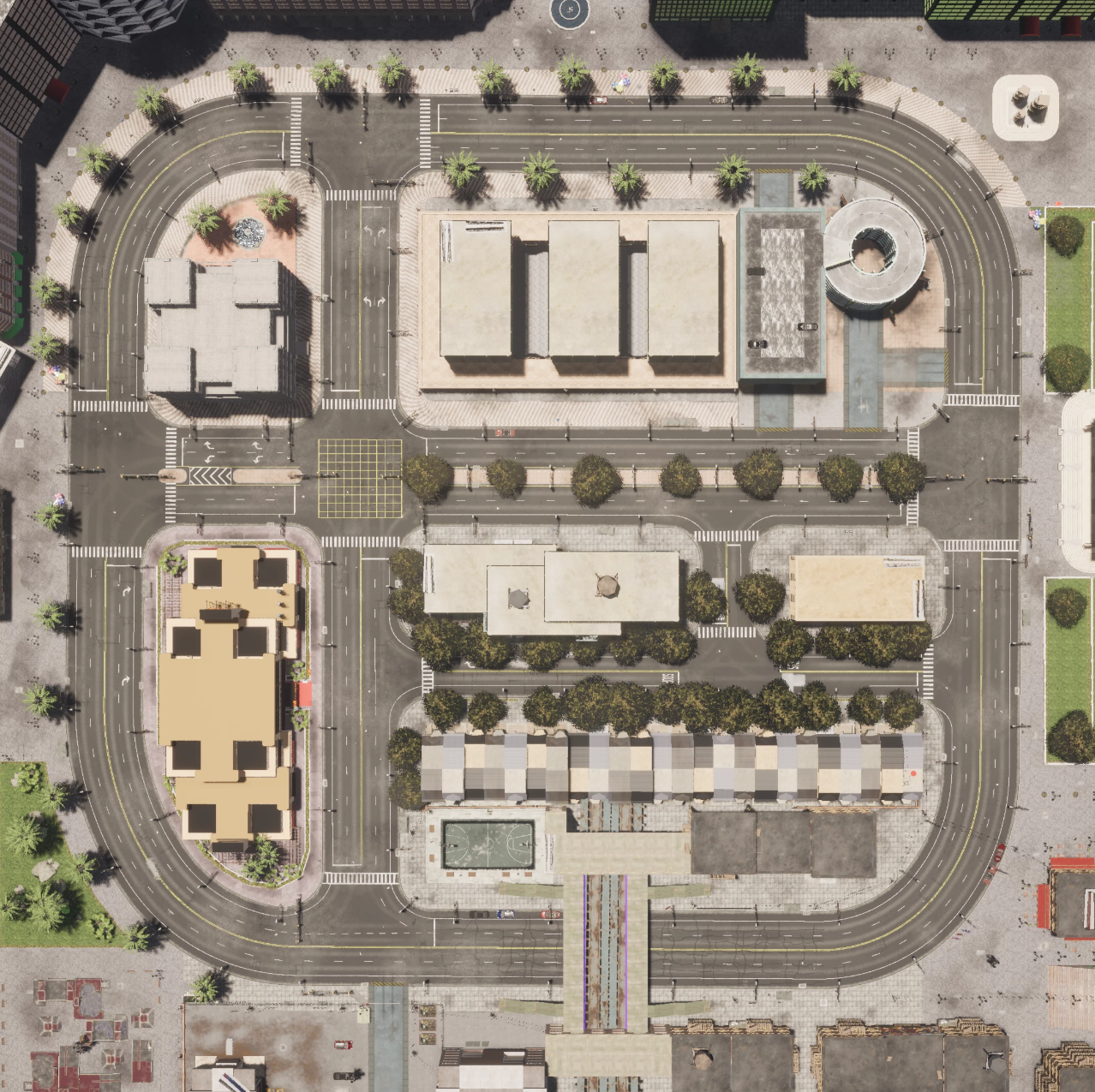}
            \label{fig:app:carla_map}
        } 
        \subfigure[Image input]{
            \includegraphics[width=0.46\linewidth]{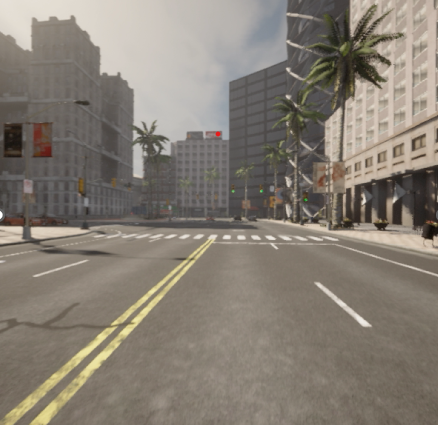}
            \label{fig:app:carla_input}
        }
        \caption{
        Visualization of CARLA map and image input
        }
    \label{fig:app:carla}
\end{figure}
\section{Implementation Details}
In this section, we describe how we generate expert datasets, and show the implementation details of $\ourmodel$ and other baselines with hyperparameter settings used for training.
For all experiments, we use a system of an NVIDIA RTX 3090 GPU and an Intel(R) Core(TM) i9-10900K CPU.

\subsection{Generating Expert Datasets and Oracle}
To generate expert datasets, we emulate expert using PGMORL algorithm~\cite{morl:pgmorl}, which is a state-of-the-art multi-objective RL method. 
%
The implementation of PGMORL is based on the open source project\footnote{\url{https://github.com/mit-gfx/PGMORL}}. 
Using PGMORL on the ground truth reward functions, we are able to collect multiple datasets with different preferences with sufficient diversity.
Among these datasets, we use only two distinct datasets, each associated with a specific preference over multi-objectives.
Regarding to the oracle, we utilize complete datasets generated by PGMORL to measure the performance.
For hyperparameter settings, we use the default settings listed in the PGMORL project.

To emulate different experts for Carla, we use heuristic agent provided by the Carla~\cite{carla}.
By varying the maximum velocity of the agent can reach, we obtain distinct expert datasets.
\subsection{DiffBC}
We implement DiffBC using the denoising diffusion probabilistic model (DDPM)~\cite{il:dubc} with augmented datasets.
%
%
We linearly sample the preference weights $\omega_i \in [0,1]$ where $\sum\nolimits_i \omega_i = 1$, and we train the algorithm with different preference weights multiple times to obtain the approximated Pareto policy set $\PolicySet$, which contains the same number of policies as $\ourmodel$. 
This augmenting method is consistent throughout the baselines.
The hyperparameter settings for DiffBC are summarized in Table~\ref{tbl:bc_params}.
\begin{table}[h]
    \small
    \centering
    \begin{tabular}{lc}
    \toprule
    \textbf{HyperParameter} & \textbf{Value} \\
    \midrule
    Learning rate       & $3 \times 10^{-4}$ \\
    Batch size          & $256$ \\
    Timestep            & $1 \times 10^5$ \\
    Actor network       & $[64,64]$ \\
    Activation function & gelu \\
    Total denoise timestep & $20$ \\
    Variance Scheduler  & cosine \\
    \bottomrule
    \end{tabular}
    \caption{Hyperparameter settings for DiffBC}
    \label{tbl:bc_params}
\end{table}

\subsection{BeT}
We implement BeT using the open source project\footnote{\url{https://github.com/notmahi/bet?tab=readme-ov-file}} with augmented datasets. 
BeT employs the transformer architecture with an action discretization and a multi-task action correction, which allows to effectively learn the multi-modality present in the datasets.
The hyperparameter settings for BeT are summarized in Table~\ref{tbl:bet_params}.
\begin{table}[h]
    \small
    \centering
    \begin{tabular}{lc}
    \toprule
    \textbf{HyperParameter} & \textbf{Value} \\
    \midrule
    Learning rate       & $3 \times 10^{-5}$ \\
    Batch size          & $64$ \\
    Timestep            & $1 \times 10^5$ \\
    Layer Size          & $[256,256]$ \\
    Number of Heads     & $4$ \\
    Activation function & gelu \\
    Number of clusters  & $32$ \\
    History length      & $4$ \\
    \bottomrule
    \end{tabular}
    \caption{Hyperparameter settings for BeT}
    \label{tbl:bet_params}
\end{table}
\subsection{GAIL and AIRL}
We implement GAIL and AIRL using the open source projects Jax\footnote{\url{https://github.com/google/jax}} 
and Haiku\footnote{\url{https://github.com/deepmind/dm-haiku}} with augumented datasets.
These algorithms are structured with a discriminator involving a reward approximator, a shaping term, and a generator (policy).
For the generator, we use the PPO algorithm. 
%
For better convergences, we pretrain a policy with BC for fixed timesteps.
The hyperparameter settings for the generator and discriminator are summarized in Table~\ref{tbl:gen_params} and Table~\ref{tbl:disc_params}, respectively.
\begin{table}[h]
\centering
    \small
    \begin{tabular}{lc}
    \toprule
    \textbf{HyperParameter} & \textbf{Value} \\
    \midrule
    Learning rate       & $3 \times 10^{-4}$ \\
    Epoch               & $50$ \\
    Batch size          & $16384$ \\
    Clip range          & $0.1$ \\
    Max grad norm       & $0.5$ \\
    \multirow{2}{*}[0pt]{Timestep} 
    & $3 \times 10^6$ (MuJoCo) \\
    & $1 \times 10^5$ (others) \\
    Actor network       & $[64,64]$ \\ 
    Value network       & $[64,64]$ \\
    Activation function & tanh \\
    \bottomrule
    \end{tabular}
    \caption{Hyperparameter settings for generator (PPO)}
    \label{tbl:gen_params}
\end{table}
\begin{table}[h]
\centering
    \small
    \begin{tabular}{lc}
    \toprule
    \textbf{HyperParameter} & \textbf{Value} \\
    \midrule
    Learning rate       & $3 \times 10^{-4}$ \\
    Epoch               & $50$ \\
    Batch size          & $64$ \\
    Reward network      & $[32]$ \\
    Shaping network     & $[32,32]$ \\
    Activation function & relu \\
    \bottomrule
    \end{tabular}
    \caption{Hyperparameter settings for discriminator}
    \label{tbl:disc_params}
\end{table}

\subsection{IQ-Learn}
We implement IQ-Learn using the open source project\footnote{\url{https://github.com/Div99/IQ-Learn}} with augmented datasets.
IQ-Learn employs a single Q-function to implicitly represent a reward function and a policy.
For learning, we us the SAC algorithm.
For hyperparameter settings, we use the default settings listed in the IQ-Learn project~\cite{irl:iqlearn}.

\subsection{DiffAIL}
We implement DiffAIL using the open source project\footnote{\url{https://github.com/ML-Group-SDU/DiffAIL}} with augmented datasets.
DiffAIL is structured with a discriminator and a geneartor (policy), where diffusion loss is incorporated into the discriminator objective.
For better convergence, we also pretrain a policy with BC as GAIL and AIRL.
For hyperparameter settings, we use the default settings listed in the DiffAIL project~\cite{irl:diffail}.

\subsection{$\ourmodel$}
We implement the entire procedure of our $\ourmodel$ framework exploiting the open source projects Jax\footnote{\url{https://github.com/google/jax}} 
and Haiku\footnote{\url{https://github.com/deepmind/dm-haiku}}.
We use the same hyperparameter settings for the generator (Table~\ref{tbl:gen_params}) and the discriminator (Table~\ref{tbl:disc_params}).
In addition, the same pre-training method and training timesteps are adopted for learning the individual IRL procedure (the first step) in our $\ourmodel$ framework.  
For each recursive step $g \geq 2$, we use the previously derived policy and the discriminator as the initial point for the IRL procedure of the current $g$. 
This reduces the training time of IRL procedure at each recursive step to at most ${1}/{30}$ of the individual IRL procedure at the first step. 

Regarding the reward regularization loss, we set the hyperparameter $\beta$ to $9$, and we sample the same size of batches across multiple datasets $\{\Dataset_i^g\}_{i=1}^{M}$ to calculate the reward distance.
To canonicalize the rewards, we sample $S,S'$ and $A$ with size of $ 512$ independently from uniform distributions. 
The hyperparameters for $\ourmodel$ are summarized in Table~\ref{tbl:piirl_params}. 
\begin{table}[h]
    \centering
    \small
    \begin{tabular}{lccccc}
    \toprule
    \textbf{Environment} & \textbf{Steps} & \textbf{Timestep} & \textbf{$\beta$} & \textbf{Batch size} \\
    \midrule
    MO-Car   & $6$     & $2.5 \times 10^4$   & $9$   & $128,\ 512$ \\
    MO-Swimmer      & $6$     & $1 \times 10^5$     & $9$   & $128,\ 512$ \\
    MO-Cheetah      & $11$    & $4 \times 10^5$     & $9$   & $128,\ 512$ \\
    MO-Ant          & $11$    & $8 \times 10^5$     & $9$   & $128,\ 512$ \\
    MO-AntXY        & $11$    & $4 \times 10^5$     & $9$   & $128,\ 512$ \\
    MO-Car*         & $6$     & $1.5 \times 10^4$   & $9$   & $128,\ 512$ \\
    CARLA        & $3$     & $3.5 \times 10^4$   & $9$   & $128,\ 512$ \\
    \bottomrule
    \end{tabular}
    \caption{Hyperparameter settings for $\ourmodel$}
    \label{tbl:piirl_params}
\end{table}
%

To train a preference-conditioned diffusion policy, we linearly match the preference weights to the policies learned through our $\ourmodel$.
Then we sample 32 size of batches each from different datasets which are collected from policies to train a preference-conditioned policy.
These policies do not emulate the experts of given datasets, but are used to augment the given datasets with imaginary experts of more diverse preferences.
For evaluation, we arbitrary sample the preference $\omega_i \in [0,1]$, where $\sum\nolimits_i \omega_i = 1$.
The hyperparameter settings for the preference-conditioned policy are summarized in Table~\ref{tbl:pc_params}.
\begin{table}[h]
    \small
    \centering
    \begin{tabular}{lc}
    \toprule
    \textbf{HyperParameter} & \textbf{Value} \\
    \midrule
    Learning rate       & $3 \times 10^{-5}$ \\
    Batch size per demo & $32$ \\
    Timestep            & $2 \times 10^5$ \\
    Actor network       & $[256,256]$ \\
    Activation function & gelu \\
    Total denoise timestep & $50$ \\
    Variance Scheduler  & cosine \\
    Guidance weight ($\delta$) & $1.2$ \\
    \bottomrule
    \end{tabular}
    \caption{Hyperparameter settings for $\ourmodel$+DU}
    \label{tbl:pc_params}
\end{table}

\section{Pareto Frontier Visualization}
Figure~\ref{fig:app:pareto} shows the Pareto frontier all acquired by our framework ($\ourmodel$, $\ourmodel$+DU) and other baselines (DiffBC, BeT, GAIL, AIRL, IQ-Learn, DiffAIL), for each environment. 
For each case, we conduct the experiments with 3 random seeds and visualize the best results regarding HV.
Note that we exclude out-of-order policies obtained from algorithms with respect to preferences, thus, the number of dots are different for the baselines and our framework.
As shown, $\ourmodel$ and $\ourmodel$+DU render the competitive Pareto Frontier of the most densely populated policies compared to the baselines.
\begin{figure*}[t]
    \centering
    \subfigure[MO-CarControl]{
        \begin{adjustbox}{width=1\textwidth}
        \includegraphics[width=0.2\textwidth]{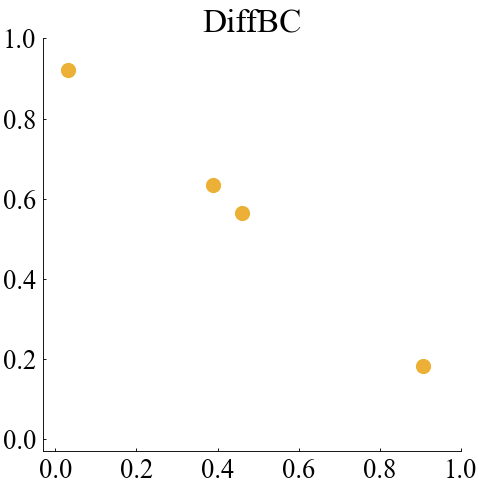}
        \includegraphics[width=0.2\textwidth]{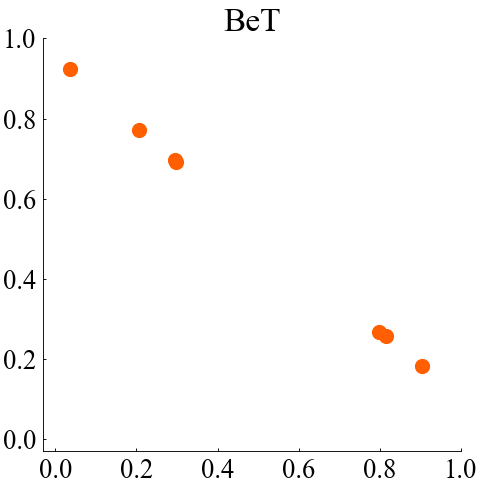}
        \includegraphics[width=0.2\textwidth]{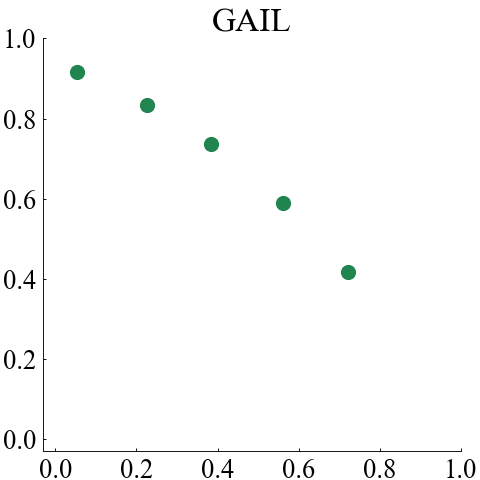}
        \includegraphics[width=0.2\textwidth]{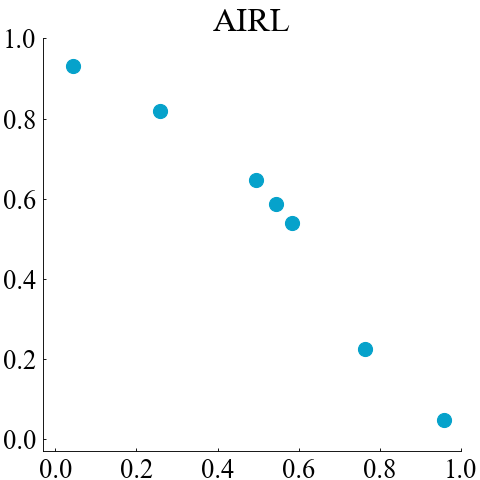}
        \includegraphics[width=0.2\textwidth]{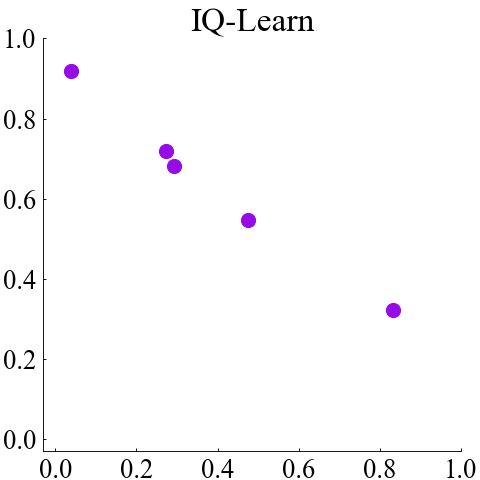}
        \includegraphics[width=0.2\textwidth]{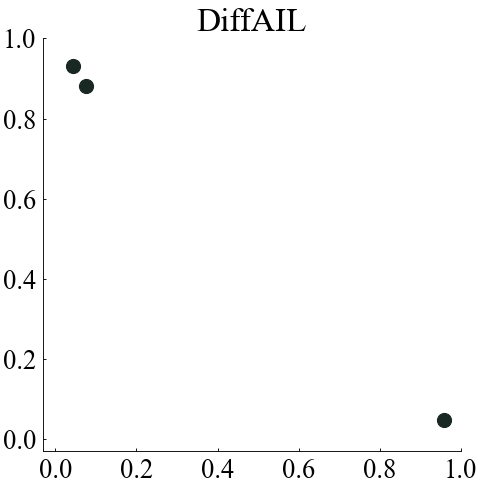}
        \includegraphics[width=0.2\textwidth]{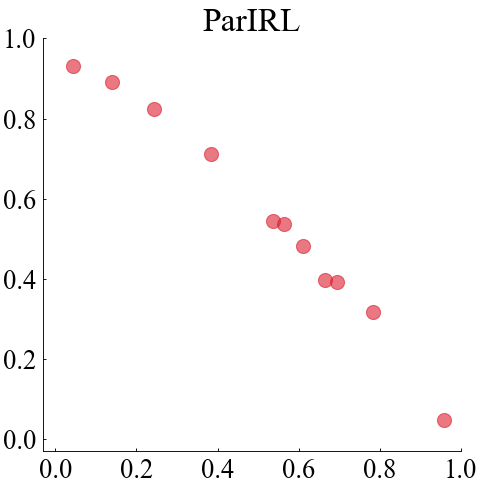}
        \includegraphics[width=0.2\textwidth]{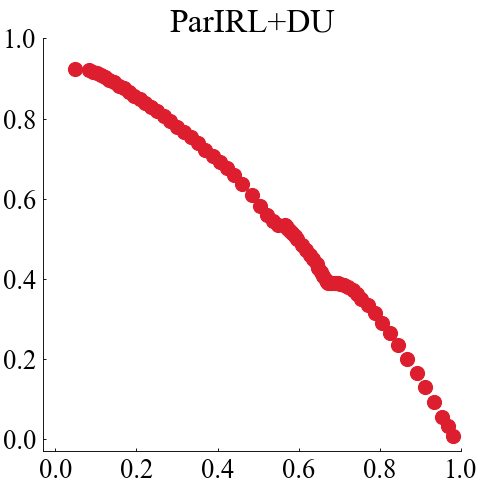}
        \end{adjustbox}
    }
    \subfigure[MO-Swimmer]{
        \begin{adjustbox}{width=1\textwidth}
        \includegraphics[width=0.2\textwidth]{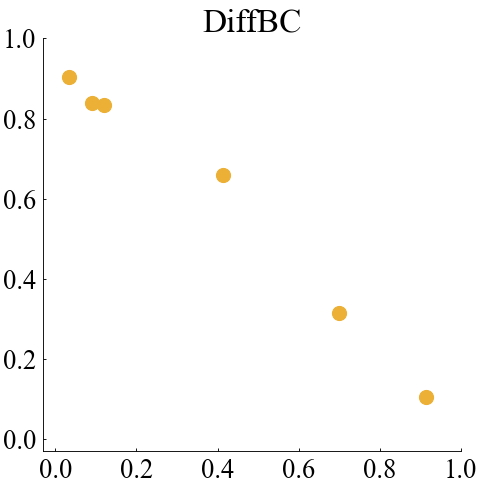}
        \includegraphics[width=0.2\textwidth]{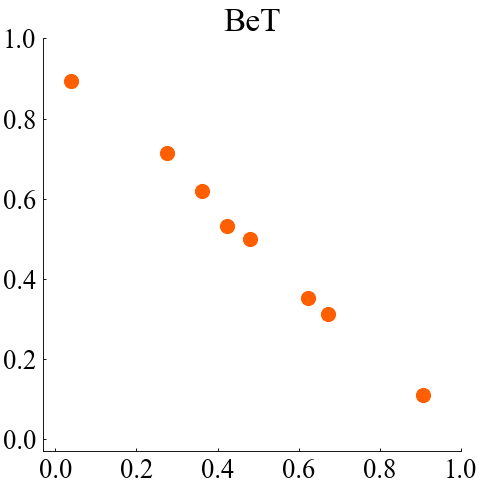}
        \includegraphics[width=0.2\textwidth]{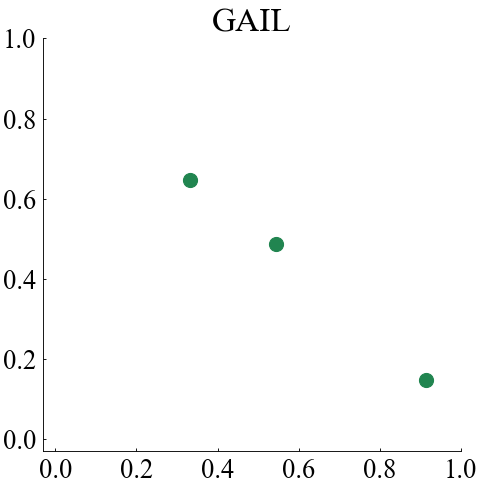}
        \includegraphics[width=0.2\textwidth]{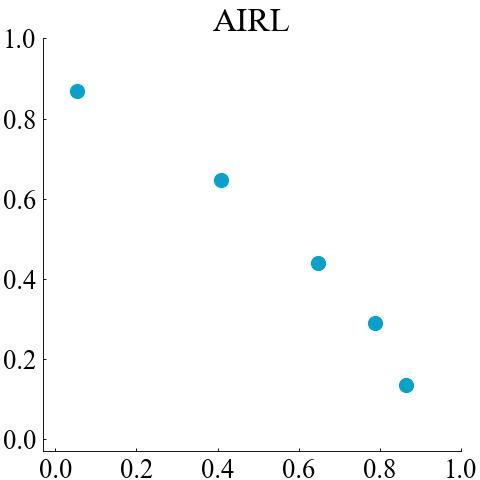}
        \includegraphics[width=0.2\textwidth]{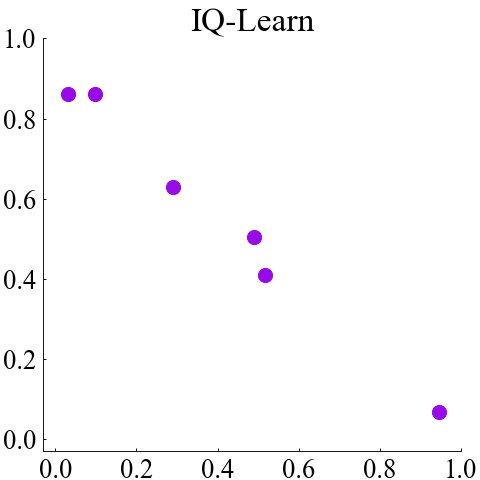}
        \includegraphics[width=0.2\textwidth]{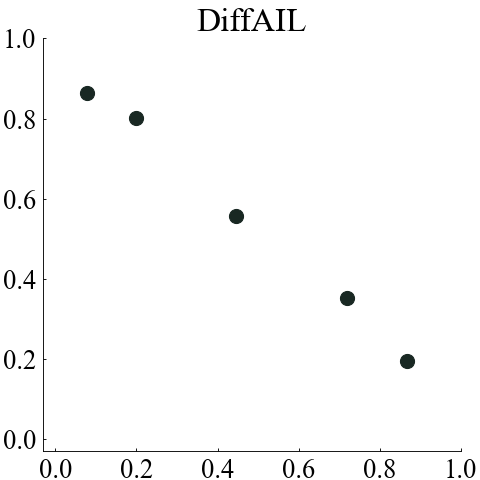}
        \includegraphics[width=0.2\textwidth]{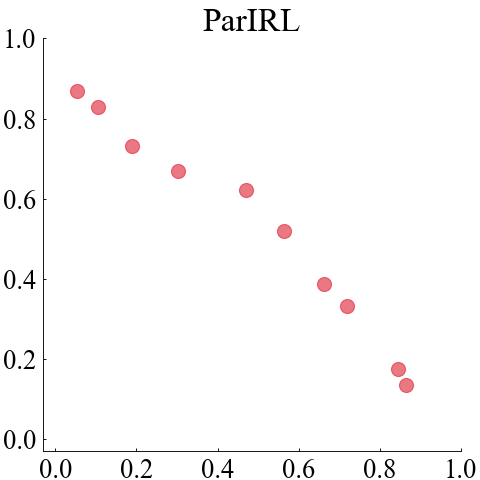}
        \includegraphics[width=0.2\textwidth]{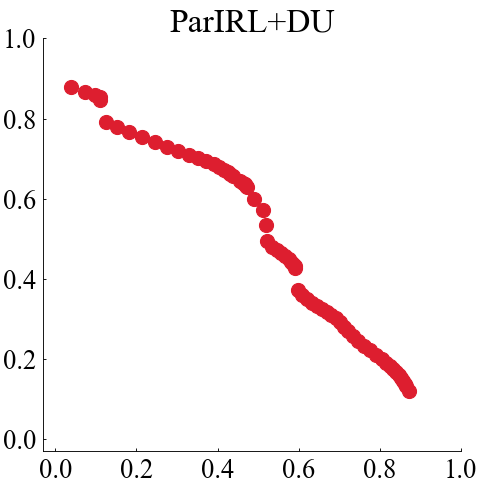}
        \end{adjustbox}
    }
    \subfigure[MO-Cheetah]{
        \begin{adjustbox}{width=1\textwidth}
        \includegraphics[width=0.2\textwidth]{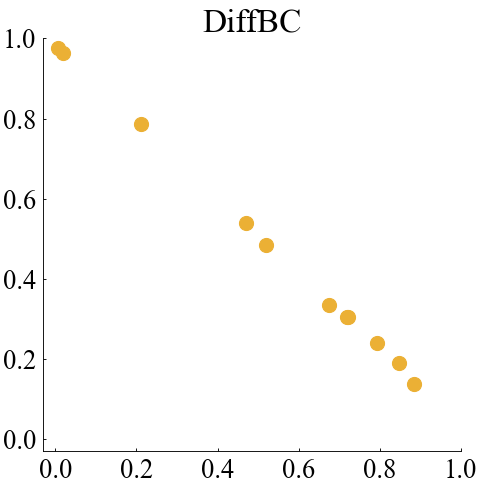}
        \includegraphics[width=0.2\textwidth]{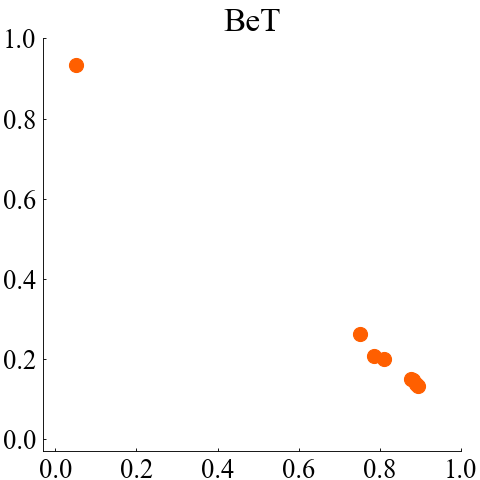}
        \includegraphics[width=0.2\textwidth]{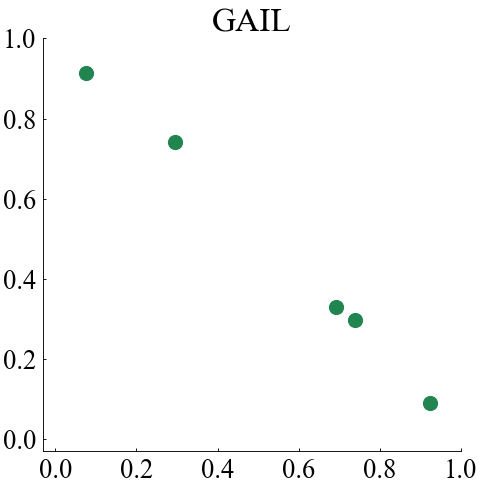}
        \includegraphics[width=0.2\textwidth]{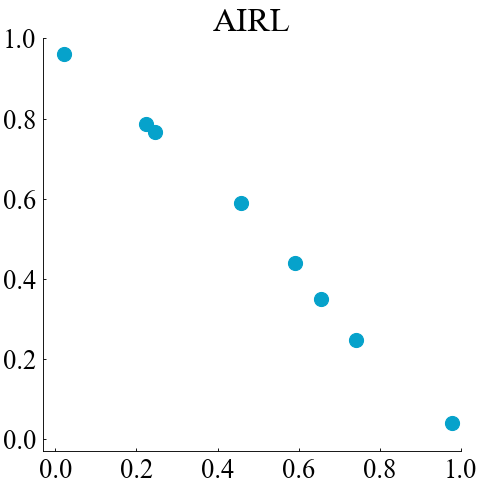}
        \includegraphics[width=0.2\textwidth]{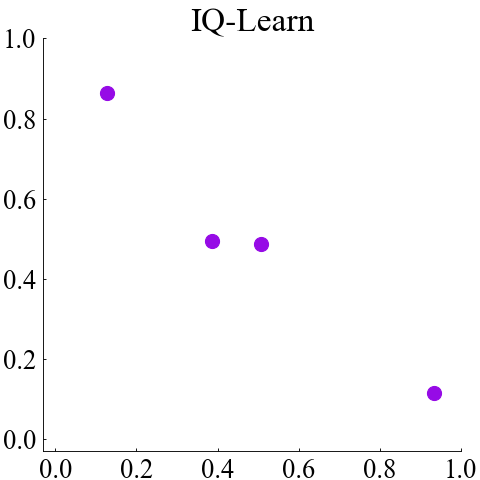}
        \includegraphics[width=0.2\textwidth]{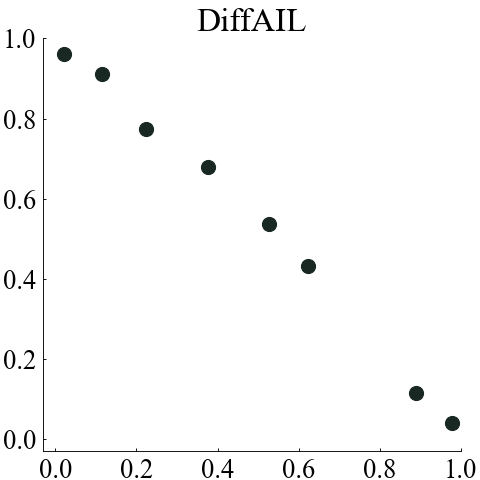}
        \includegraphics[width=0.2\textwidth]{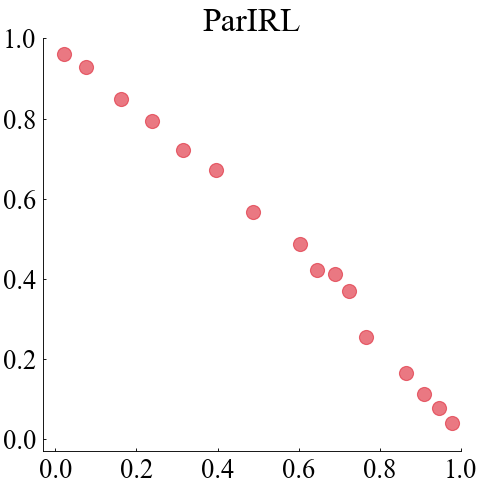}
        \includegraphics[width=0.2\textwidth]{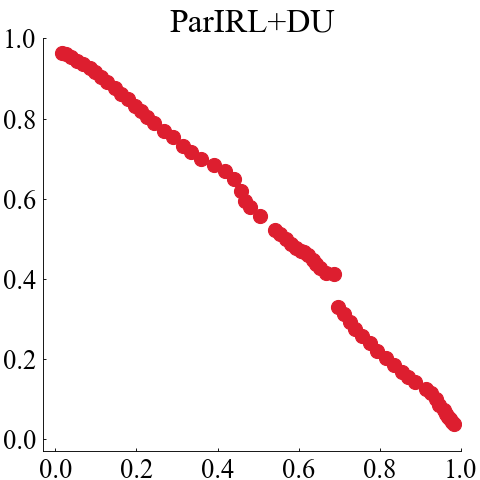}
        \end{adjustbox}
    }
    \subfigure[MO-Ant]{
        \begin{adjustbox}{width=1\textwidth}
        \includegraphics[width=0.2\textwidth]{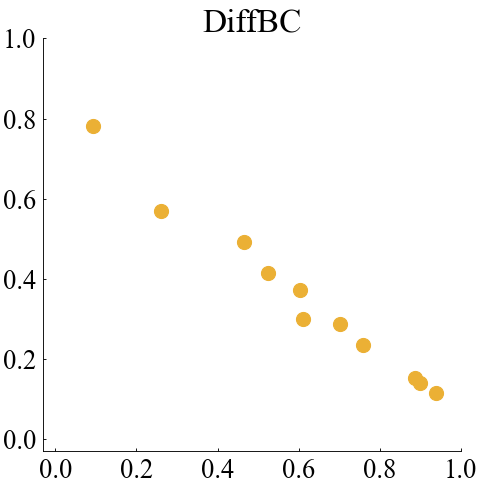}
        \includegraphics[width=0.2\textwidth]{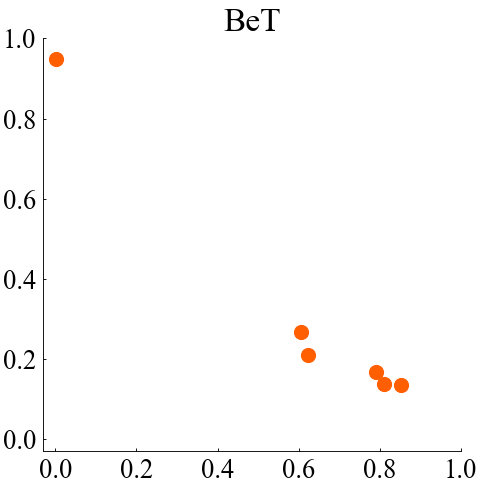}
        \includegraphics[width=0.2\textwidth]{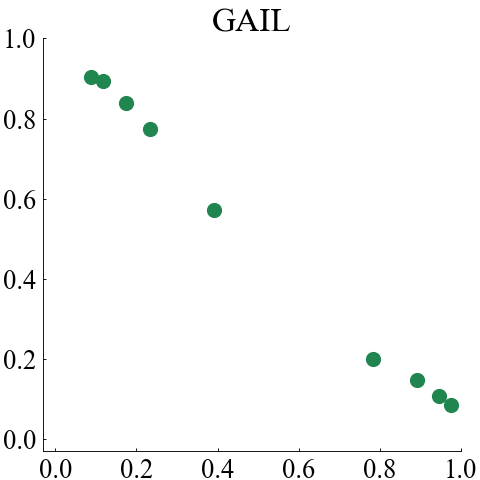}
        \includegraphics[width=0.2\textwidth]{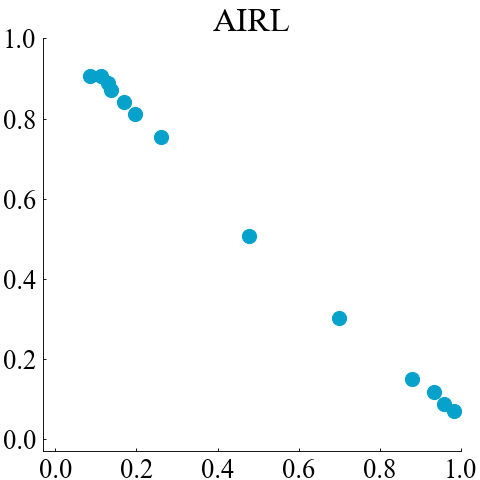}
        \includegraphics[width=0.2\textwidth]{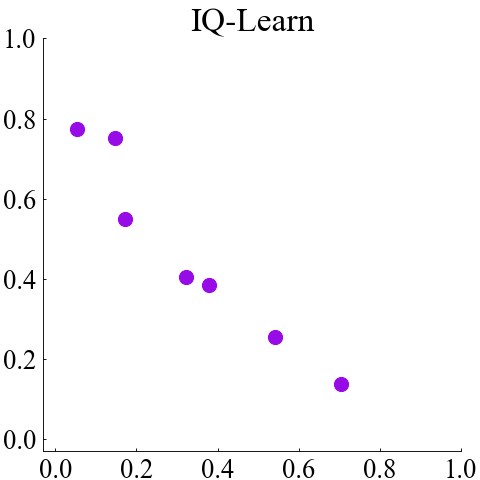}
        \includegraphics[width=0.2\textwidth]{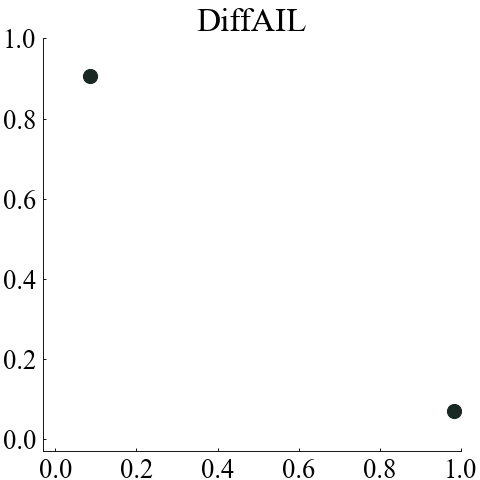}
        \includegraphics[width=0.2\textwidth]{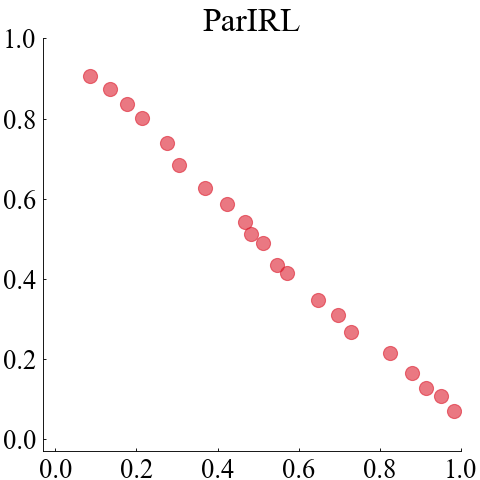}
        \includegraphics[width=0.2\textwidth]{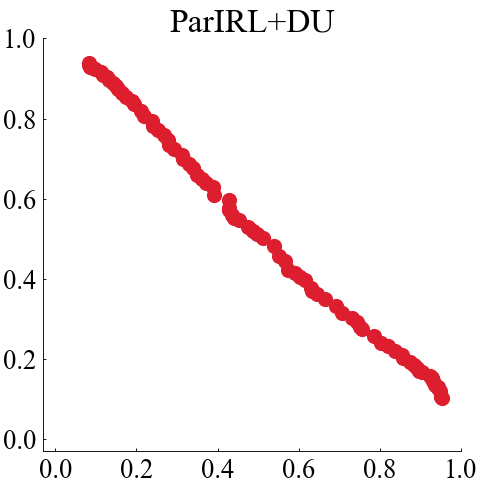}
        \end{adjustbox}
    }
    \subfigure[MO-AntXY]{
        \begin{adjustbox}{width=1\textwidth}
        \includegraphics[width=0.2\textwidth]{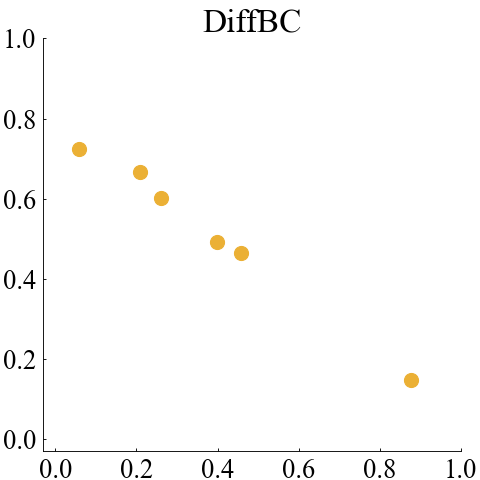}
        \includegraphics[width=0.2\textwidth]{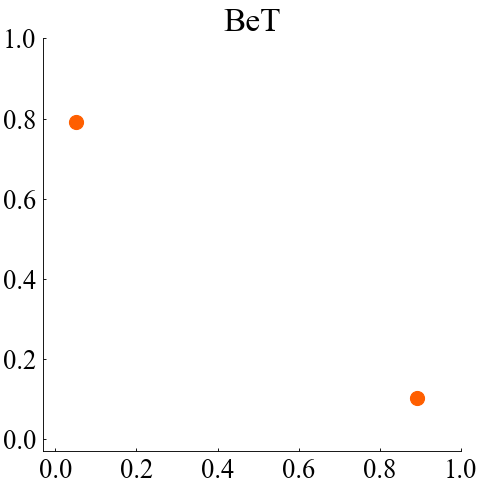}
        \includegraphics[width=0.2\textwidth]{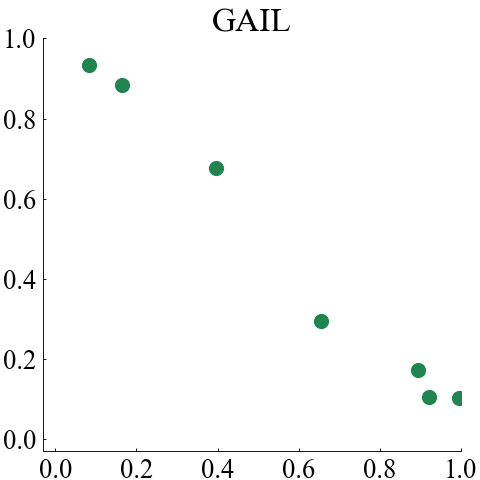}
        \includegraphics[width=0.2\textwidth]{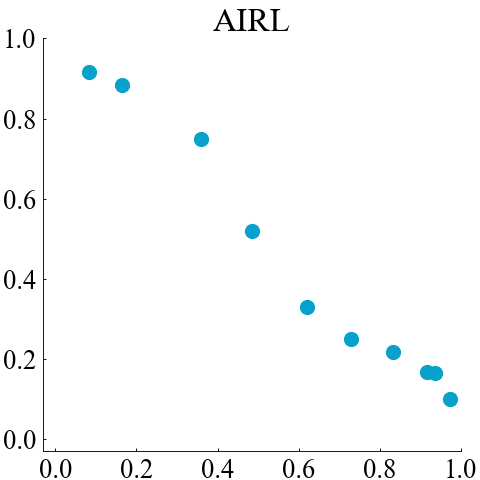}
        \includegraphics[width=0.2\textwidth]{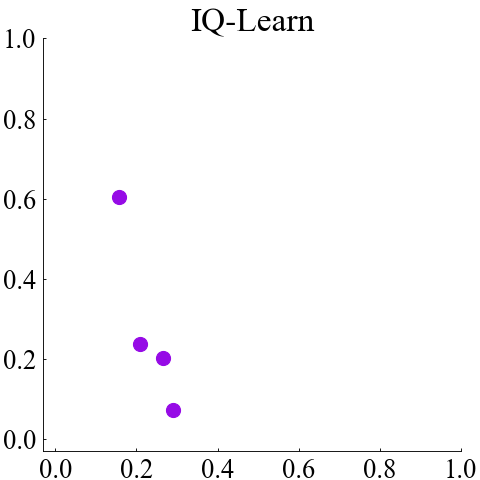}
        \includegraphics[width=0.2\textwidth]{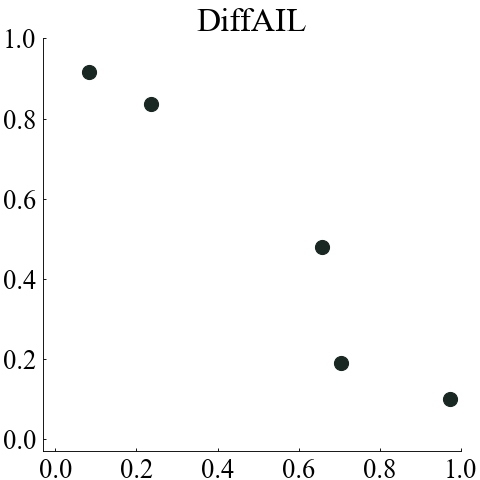}
        \includegraphics[width=0.2\textwidth]{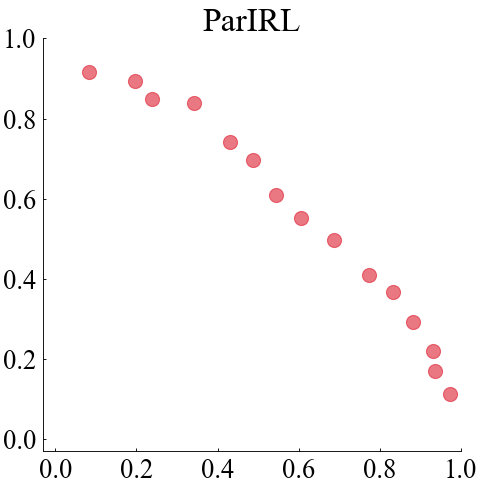}
        \includegraphics[width=0.2\textwidth]{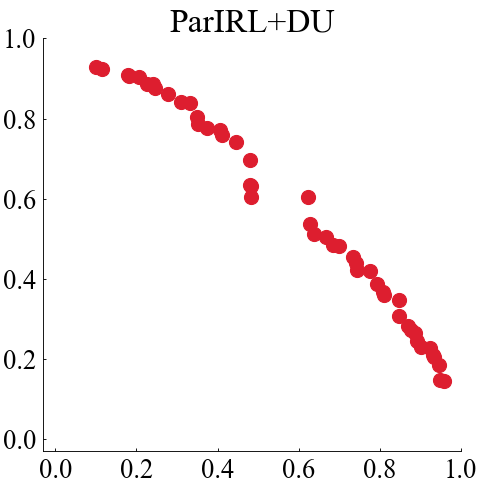}
        \end{adjustbox}
    }
    \caption{Pareto frontier visualization: the graphs in one line are for the same environment, and BC, BeT, GAIL, AIRL, IQ-Learn, DiffAIL, $\ourmodel$, and $\ourmodel$+DU are each located from the left. Note that we exclude out-of-order policies obtained from algorithms with respect to preferences.}
    \label{fig:app:pareto}
\end{figure*}